%% file: arxiv_main.tex
\documentclass[11pt]{article}
\input{preamble}

\title{Composite Online-to-Nonconvex Conversion\\
with Optimal Oracle Complexity}
\author{
  Mingyi Li\footnote{
    The University of Tokyo;
    \texttt{mingyi-mike@g.ecc.u-tokyo.ac.jp}.
  }
  \and
  Taira Tsuchiya\footnote{
    The University of Tokyo, The University of Osaka, and RIKEN;
    \texttt{tsuchiya@mist.i.u-tokyo.ac.jp}.
  }
  \and
  Kenji Yamanishi\footnote{
    The University of Tokyo;
    \texttt{yamanishi@g.ecc.u-tokyo.ac.jp}.
  }
}
\date{October 8, 2026}

\begin{document}
\maketitle

\input{section/abstract}

\input{section/01_introduction}
\input{section/02_related_work}
\input{section/03_preliminaries}
\input{section/04_proposed_algorithm}

\input{section/05_first_order}

\input{section/06_zeroth_order}
\input{section/07_smooth_setting}
\input{section/08_numerical_figures}
\input{section/08_numerical_experiments}

\input{section/ai_use_statement}

\section*{Acknowledgements}
TT is supported by JSPS KAKENHI Grant Number JP26K21297 and JST BOOST Grant Number JPMJBY25D4,
KY is partially supported by JSPS KAKENHI Grant Number JP24H00703.

\bibliographystyle{plainnat}
\bibliography{reference}

\clearpage
\appendix
\crefalias{section}{appendix}
\crefalias{subsection}{appendix}
\input{section/appendix}

\end{document}

%% file: preamble.tex
\pdftrailerid{}

\usepackage[top=30truemm,bottom=30truemm,left=25truemm,right=25truemm]{geometry}
\usepackage{parskip}
\usepackage[utf8]{inputenc}
\usepackage[T1]{fontenc}
\usepackage{tgtermes}
\usepackage{microtype}
\usepackage{amsmath}
\usepackage{amssymb}
\usepackage{bbm}
\usepackage{amsthm}
\usepackage{mathtools}

\usepackage[authoryear,round]{natbib}

\usepackage[ruled,vlined,linesnumbered]{algorithm2e}

\usepackage{xcolor}

\usepackage[hypertexnames=false,backref,pagebackref]{hyperref}
\hypersetup{
  pdfcreator={},
  pdfproducer={},
  colorlinks=true,
  linkcolor=blue,
  citecolor=blue,
}
\usepackage[nameinlink,capitalise,noabbrev]{cleveref}
\newcommand{\fixcrefthm}[1]{%
  \AddToHook{env/#1/begin}{\crefalias{theorem}{#1}}%
}
\fixcrefthm{theorem}
\fixcrefthm{lemma}
\fixcrefthm{proposition}
\fixcrefthm{corollary}
\fixcrefthm{assumption}
\fixcrefthm{definition}
\fixcrefthm{remark}

\crefname{equation}{}{}
\crefname{theorem}{Theorem}{Theorems}
\Crefname{theorem}{Theorem}{Theorems}
\crefname{lemma}{Lemma}{Lemmas}
\Crefname{lemma}{Lemma}{Lemmas}
\crefname{corollary}{Corollary}{Corollaries}
\Crefname{corollary}{Corollary}{Corollaries}
\crefname{definition}{Definition}{Definitions}
\Crefname{definition}{Definition}{Definitions}
\crefname{remark}{Remark}{Remarks}
\Crefname{remark}{Remark}{Remarks}
\crefname{assumption}{Assumption}{Assumptions}
\Crefname{assumption}{Assumption}{Assumptions}
\crefname{proposition}{Proposition}{Propositions}
\Crefname{proposition}{Proposition}{Propositions}
\Crefname{algocf}{Algorithm}{Algorithms}
\crefname{line}{Line}{Lines}
\Crefname{line}{Line}{Lines}
\renewcommand{\nllabel}[1]{\addtocounter{AlgoLine}{-1}\refstepcounter{AlgoLine}\label[line]{#1}}
\crefname{appendix}{Appendix}{Appendices}
\Crefname{appendix}{Appendix}{Appendices}

\SetArgSty{textnormal}
\SetKwInOut{KwRequire}{Require}

\usepackage{graphicx}
\graphicspath{{figure/}}
\usepackage{enumitem}
\usepackage{booktabs}
\usepackage{array}
\usepackage{pdflscape}

\usepackage{wrapfig}
\usepackage{tikz}

\theoremstyle{plain}
\newtheorem{theorem}{Theorem}[section]
\newtheorem{lemma}[theorem]{Lemma}
\newtheorem{proposition}[theorem]{Proposition}

\theoremstyle{definition}
\newtheorem{definition}[theorem]{Definition}
\newtheorem{assumption}[theorem]{Assumption}
\theoremstyle{remark}

\newcommand{\calC}{\mathcal{C}}

\newcommand{\calF}{\mathcal{F}}

\newcommand{\calU}{\mathcal{U}}

\renewcommand{\epsilon}{\varepsilon}

\def\E{\mathbb{E}}

\newcommand{\ind}{\mathbbm{1}}
\def\R{\mathbb{R}}

\DeclareMathOperator*{\argmin}{arg\,min}

\DeclarePairedDelimiter{\abs}{\lvert}{\rvert}
\DeclarePairedDelimiter{\inpr}{\langle}{\rangle}
\DeclarePairedDelimiter{\brk}{[}{]}
\DeclarePairedDelimiter{\prn}{(}{)}
\DeclarePairedDelimiter{\set}{\{}{\}}
\DeclarePairedDelimiter{\nrm}{\|}{\|}
\DeclarePairedDelimiter{\ceil}{\lceil}{\rceil}
\newcommand{\unif}{\mathsf{Unif}}
\renewcommand{\hat}{\widehat}
\newcommand{\dd}{\mathrm{d}}
\newcommand{\CompositeONC}{\textsc{Composite O2NC}}

\allowdisplaybreaks[1]
\usepackage{autonum}

\SetKwInput{Require}{Require}
\SetKwInput{Return}{Return}

%% file: section/abstract.tex
\begin{abstract}
We consider stochastic nonsmooth nonconvex composite optimization, which includes several important problems such as constrained optimization and the regularized training of neural networks.  The objective is the sum of a possibly nonsmooth nonconvex Lipschitz function and a convex regularizer, and the function is accessed through stochastic gradients or function values. The goal is to find a point that satisfies a Goldstein-type stationarity condition designed for composite objectives. To our knowledge, no oracle complexity bound for this setting is known under first-order access, and existing complexities under zeroth-order access are suboptimal. To handle this issue, we employ the framework of online-to-nonconvex conversion, which chooses update directions by an online learner and is known to achieve optimal rates for noncomposite problems. We extend the framework to our composite scenario by introducing new losses for the learner, which contain the regularizer itself rather than its linearization and for which a variant of online mirror descent achieves low regret. We show that the resulting algorithm finds such a point with $O(\delta^{-1}\varepsilon^{-3})$ stochastic gradient queries or $O(d\delta^{-1}\varepsilon^{-3})$ function-value queries, where $\delta$ is the Goldstein radius, $\varepsilon$ is the stationarity tolerance, and $d$ is the dimension. These rates match the optimal ones for noncomposite nonsmooth nonconvex optimization, demonstrating that the additional convex regularizer does not worsen the oracle complexity. We also give rates for the smooth case and present numerical experiments.
\end{abstract}

%% file: section/01_introduction.tex
\section{Introduction}\label{sec:intro}
\begin{table*}[t]
  \centering
  \small
  \caption{Oracle complexity for finding a $(\gamma,\delta,\epsilon)$-PGSP (see \cref{def:pgsp}), which reduces to $(\delta,\epsilon)$-Goldstein stationarity when $h=0$.
  The bounds show the dependence on $d$, $\delta$, and $\epsilon$ for $\delta,\epsilon\in(0,1]$.}
  \label{tab:nonsmooth-comparison}
  \begin{tabular}{llll}
    \toprule
    Reference & Problem class & Oracle & Complexity \\
    \midrule
    \citet{cutkosky2023}
      & $h=0$ & First-order & $O(\delta^{-1}\epsilon^{-3})$ \\
    \citet{kornowski2024zo}
      & $h=0$ & Zeroth-order & $O(d\delta^{-1}\epsilon^{-3})$ \\
    \citet[Theorem~2]{chen2025}
      & Composite & Zeroth-order (minibatch) & $O(d^{3/2}\delta^{-1}\epsilon^{-4})$ \\
    \citet[Theorem~3]{chen2025}
      & Composite & Zeroth-order (variance reduction) & $O(d^{3/2}\delta^{-1}\epsilon^{-3})$ \\
    \citet{pougkakiotis2026inexact}
      & Composite & Zeroth-order & $O(d^{3/2}\delta^{-1}\epsilon^{-4})$ \\
    \midrule
    \textbf{This work} (\cref{thm:first-order-complexity})
      & Composite & First-order & $O(\delta^{-1}\epsilon^{-3})$ \\
    \textbf{This work} (\cref{thm:zeroth-order-complexity})
      & Composite & Zeroth-order & $O(d\delta^{-1}\epsilon^{-3})$ \\
    \bottomrule
  \end{tabular}
\end{table*}

We consider stochastic nonsmooth nonconvex composite optimization,
\begin{align}
  \min_{x\in\R^d} \set*{ \Phi(x)\coloneqq f(x)+h(x) },
  \ 
  f(x)\coloneqq\E_\xi\brk*{F(x,\xi)},
  \label{eq:intro-problem}
\end{align}
where $\xi$ is a random vector, $f\colon\R^d\to\R$ is locally Lipschitz and may be nonconvex and nonsmooth, and $h\colon\R^d\to\R\cup\set{+\infty}$ is a proper closed convex function.
We study both first-order access to $f$ through stochastic gradients and zeroth-order access through stochastic function values.

Composite optimization includes important problem classes such as $\ell_1$-regularized optimization, nuclear-norm-regularized optimization, and constrained optimization~\citep{tibshirani1996lasso,mazumder2010spectral,parikh2014proximal}.
Most existing convergence analyses of composite optimization assume that $f$ is smooth, that is, $\nabla f$ is Lipschitz continuous~\citep{nesterov2013composite,ghadimi2016minibatch}.
In this smooth setting, stochastic first-order and zeroth-order methods find an $\epsilon$-stationary point of the objective $\Phi$ with $O(\epsilon^{-4})$ stochastic gradient queries and $O(d\epsilon^{-4})$ function-value queries, respectively~\citep{ghadimi2016minibatch}, where $\epsilon$ is the stationarity tolerance.
In many applications, however, $f$ is nonsmooth, as in support vector machines with the hinge loss~\citep{cortes1995support,shalev2011pegasos} and the regularized training of neural networks with ReLU activations~\citep{krizhevsky2012imagenet,wang2022hidden}.

With zeroth-order access, \citet{liu2024constrained} study nonsmooth nonconvex optimization over compact convex sets, and \citet{chen2025} extend this setting to composite objectives.
Their algorithms require $O(d^{3/2}\delta^{-1}\epsilon^{-4})$ function-value queries with minibatch estimation to find proximal Goldstein stationary points (PGSPs; see \cref{def:pgsp}), a Goldstein-type stationarity condition for composite objectives, where $\delta$ is the Goldstein radius and $d$ is the dimension.
Incorporating variance reduction further improves the bound to $O(d^{3/2}\delta^{-1}\epsilon^{-3})$.
To our knowledge, no stochastic first-order complexity bound for PGSPs is known for general Lipschitz nonconvex $f$ and proper closed convex $h$.

In the noncomposite setting where $h=0$, the seminal framework of online-to-nonconvex conversion (O2NC), which chooses update directions by online learning, allows us to obtain the optimal $O(\delta^{-1}\epsilon^{-3})$ stochastic gradient query complexity~\citep{cutkosky2023} and $O(d\delta^{-1}\epsilon^{-3})$ function-value query complexity~\citep{kornowski2024zo}.
Both rates are attained without variance reduction, and the zeroth-order bound is linear in $d$.
This raises the following question:
\begin{center}
  \emph{Can simple first-order and zeroth-order methods attain these oracle complexities for nonsmooth nonconvex composite optimization?}
\end{center}

\subsection{Contributions of this paper}\label{sec:contributions}
We answer this question affirmatively with a new conversion from online regret to proximal Goldstein stationarity.
The resulting method, \CompositeONC{}, finds a PGSP with $O(\delta^{-1}\epsilon^{-3})$ stochastic gradient queries or $O(d\delta^{-1}\epsilon^{-3})$ function-value queries.

Both versions attain these rates without variance reduction (\cref{thm:first-order-complexity,thm:zeroth-order-complexity}).
To our knowledge, our zeroth-order method is the first to find PGSPs for general nonsmooth nonconvex composite optimization with function-value query complexity linear in $d$, improving the $d^{3/2}$ dependence of \citet{chen2025}.
For $h=0$, the first-order rate matches the joint lower bound in $\delta$ and $\epsilon$~\citep{cutkosky2023}, while the zeroth-order rate is optimal in each of $d$, $\delta$, and $\epsilon$ individually~\citep{kornowski2024zo}.\footnote{\hypersetup{citecolor=black}%
\citet[Corollary~4.2]{zhang2026joint} prove a joint $\Omega(d\delta^{-1}\epsilon^{-3})$ lower bound in the setting where queries and the gap condition are restricted to a Euclidean ball, a setting different from ours.}
Since PGSP reduces to Goldstein stationarity when $h=0$, these lower bounds also apply to our composite problem class, and our rates are optimal in the same sense.
\Cref{tab:nonsmooth-comparison} summarizes the results.

Our method modifies the losses given to the online learner in O2NC.
In O2NC, the online learner selects update directions based on the linear losses $u\mapsto\inpr{g_t,u}$, where $g_t$ is an unbiased stochastic gradient of $f$ at a query point~\citep{cutkosky2023}.
This framework does not apply to $\Phi=f+h$ directly, because $\Phi$ need not be locally Lipschitz when $h$ takes the value $+\infty$.
Moreover, linearizing $h$ ignores the structure of $h$, which is known exactly.

To deal with this, we give the learner the losses $\ell_t(u)=\inpr{g_t,u}+h(x+u)-h(x)$, where $x\in\operatorname{dom}h$ is the current point.
The term $h(x+u)-h(x)$ lets the learner account for the regularizer or constraint when choosing a direction.
For these losses, composite objective mirror descent (COMID; \citealt{duchi2010}) attains the same regret bound as for linear losses, and our conversion turns this regret bound into a PGSP guarantee (\cref{lem:shifted-conversion}).
For zeroth-order access, we use two-point gradient estimates of a smoothed version of $f$ and transfer the resulting guarantee to $f+h$ using the smoothing result of \citet{kornowski2024zo}.

Our guarantee holds for every proximal parameter $\gamma>0$ simultaneously, and neither the algorithm nor its oracle complexity depends on $\gamma$.
Existing guarantees hold for particular values of $\gamma$, which is chosen as an update step size (\textit{e.g.,}~\citealt{ghadimi2016minibatch,reddi2016proximal,pham2020proxsarah,chen2025}).
In contrast, our algorithm returns $x+u^*$, where $u^*$ minimizes the averaged online loss over the feasible directions at the selected anchor $x$, and this point satisfies the PGSP guarantee for every $\gamma>0$ (\cref{lem:shifted-conversion}).
Thus $\gamma$ is used only to state the stationarity condition.

When $f$ is smooth, we show that \CompositeONC{} with an optimistic online learner finds an $\epsilon$-stationary point using $O(\epsilon^{-4})$ stochastic gradient queries or $O(d\epsilon^{-4})$ function-value queries (\cref{thm:smooth-complexity,thm:smooth-zo-complexity}).
These rates match existing bounds under the stochastic oracle, and with deterministic gradients our result recovers the $O(\epsilon^{-2})$ proximal-gradient rate~\citep{ghadimi2016minibatch} (\cref{tab:smooth-comparison} in \cref{app:smooth-comparison}).

Numerical experiments on nonsmooth composite problems show that \CompositeONC{} attains smaller values of both the stationarity measure and the objective than the zeroth-order methods of \citet{chen2025} for the same number of function-value queries (\cref{sec:experiments}).

%% file: section/02_related_work.tex
\section{Related work}\label{sec:related-work}

\paragraph{Composite optimization.}
Proximal-gradient methods are the standard approach to composite optimization~\citep{nesterov2013composite,parikh2014proximal}.
Accelerated methods include FISTA for convex objectives~\citep{beck2009fista} and extensions to nonconvex objectives~\citep{li2015accelerated,ghadimi2016accelerated,li2017momentum}.
For smooth nonconvex $f$, minibatch proximal-gradient methods attain $O(\epsilon^{-4})$ stochastic gradient query complexity~\citep{ghadimi2016minibatch}.
Other convergence analyses rely on the Kurdyka--{\L}ojasiewicz property~\citep{attouch2013descent,bolte2014palm}.

With additional smoothness assumptions on the component or sample gradients, variance-reduced methods further improve stochastic gradient query complexity.
Representative methods include ProxSVRG and its variants~\citep{xiao2014proxsvrg,nitanda2014accelerated,reddi2016proximal,li2018proxsvrg}, as well as ProxSAGA~\citep{reddi2016proximal}.
For stochastic expectation problems, ProxSpiderBoost~\citep{wang2019spiderboost}, ProxSARAH~\citep{pham2020proxsarah}, and hybrid SARAH-SGD methods~\citep{trandinh2022hybrid} attain $O(\epsilon^{-3})$ stochastic gradient query complexity under bounded variance and additional smoothness conditions on the sample gradients.

Extensions to nonsmooth $f$ have also been widely studied.
Under weak convexity, stochastic first-order~\citep{davis2019guided,davis2019model} and zeroth-order~\citep{pougkakiotis2023weakly} methods provide Moreau-envelope stationarity guarantees.
For functional constraints with Lipschitz objective and constraint functions, \citet{jia2025firstorder} establish Fritz--John and KKT guarantees under weak convexity, while \citet{grimmer2025goldstein} obtain Goldstein-type counterparts without weak convexity using deterministic first-order oracles.
The latter setting overlaps with ours for convex constraints, but we also handle composite objectives with convex regularizers using more general stochastic first-order oracles.

For smooth finite-sum composite problems, zeroth-order variants of ProxSVRG, ProxSAGA, and SPIDER replace gradients in the corresponding first-order methods with function-value estimates~\citep{huang2019gradientfree,kazemi2024zeroth}.
Beyond the weakly convex setting discussed above, \citet{liu2024constrained} study nonsmooth objectives over compact convex sets, and \citet{chen2025} extend this setting to composite objectives.
Their minibatch and variance-reduced methods find proximal Goldstein stationary points using $O(d^{3/2}\delta^{-1}\epsilon^{-4})$ and $O(d^{3/2}\delta^{-1}\epsilon^{-3})$ function-value queries, respectively.
\citet{pougkakiotis2026inexact} also study nonsmooth nonconvex composite objectives using inexact function-value queries. Their Moreau-envelope guarantee can be converted into a proximal Goldstein stationarity guarantee with $O(d^{3/2}\delta^{-1}\epsilon^{-4})$ function-value query complexity (see \cref{app:moreau-pgsp}).
For this proximal Goldstein criterion, our zeroth-order method improves all of these bounds to $O(d\delta^{-1}\epsilon^{-3})$, which is linear in $d$, without variance reduction.

\paragraph{Online-to-nonconvex conversion.}
The online-to-nonconvex conversion (O2NC) framework was introduced by \citet{cutkosky2023} for stochastic nonsmooth nonconvex optimization. It converts online regret bounds into Goldstein stationarity guarantees with optimal stochastic gradient query complexity.
\citet{kornowski2024zo} combine the O2NC conversion with randomized smoothing to obtain a bound linear in the dimension in the zeroth-order setting.

Double optimism further improves the guarantees when both the gradient and Hessian are Lipschitz~\citep{patitucci2026double}.
Random scaling and discounted regret connect O2NC to momentum~\citep{zhang2024random} and Adam with model exponential moving averages~\citep{ahn2024adamema}. The broader online-learning perspective also explains Adam through online learning of updates~\citep{ahn2024ftrl} and yields guarantees for schedule-free SGD~\citep{ahn2025schedulefree}.

For deterministic smooth objectives, \citet{cutkosky2023} recover the optimal gradient query complexity using optimistic online learners. Using momentum within discounted O2NC, \citet{li2026muon} also attain the optimal $O(\epsilon^{-2})$ gradient query complexity for deterministic smooth objectives without an optimistic online learner, using a regret decomposition based on momentum.
Other extensions handle heavy-tailed gradient noise~\citep{liu2024heavytails,liu2026ocoheavy} and remove auxiliary random interpolation or scaling under weak convexity~\citep{ji2026derandomized}. O2NC has also been applied to nonsmooth objectives on Riemannian manifolds~\citep{sahinoglu2025riemannian}, decentralized stochastic optimization~\citep{sahinoglu2024decentralized}, and linearly constrained bilevel optimization~\citep{kornowski2024bilevel}.

We extend O2NC using composite online losses that retain $h$ without linearization.
Regret bounds from composite objective mirror descent~\citep{duchi2010} then yield proximal Goldstein stationarity, even when $h$ takes the value $+\infty$.

%% file: section/03_preliminaries.tex
\section{Preliminaries}\label{sec:preliminaries}
This section states the assumptions on the objective and the stochastic oracle and introduces the Goldstein-type stationarity measure used in our analysis.

Throughout, we write $\nrm{\cdot}$ and $\inpr{\cdot,\cdot}$ for the
 Euclidean norm and inner product on $\R^d$,
$\mathbb{B}(x,r)\coloneqq\{y\in\R^d\colon\nrm{y-x}\leq r\}$ for the closed Euclidean ball,
$\mathbb{S}^{d-1}$ for the unit sphere, $\unif(S)$ for the uniform distribution on a set $S$,
and $[n]\coloneqq\set{1,\ldots,n}$ for a positive integer $n$.
For a convex function $h$, we write $\operatorname{dom}h$ and $\partial h$ for its effective domain and subdifferential.

\subsection{Problem setting}\label{sec:setting}

We consider the stochastic nonsmooth nonconvex composite optimization problem formulated in \cref{eq:intro-problem}.
We assume that $\Delta\coloneqq\Phi(x_1)-\inf_x\Phi(x)<\infty$,
where $x_1$ is the initial point.
We study two oracle models for $f$.
Under first-order access, a query at a point $y\in\R^d$ returns a stochastic gradient of $f$ at $y$.
Under zeroth-order access, a query at $y$ returns the function value $F(y,\xi)$ for a sample $\xi$, where the same sample can be used at two points as in the two-point feedback model~\citep{agarwal2010multipoint,ghadimi2013zo,duchi2015zo,shamir2017zo,kornowski2024zo}.
We impose the following assumptions on the objective and the first-order oracle, following
\citet{cutkosky2023}.

\begin{assumption}\label{ass:oracle}
The function $f$ is differentiable.  For some $\sigma\geq0$
and $G>0$, a query to the oracle at any point $y\in\R^d$ is independent of the preceding history and returns a stochastic gradient $g$ satisfying
(i)~$\E\brk{g\mid y}=\nabla f(y)$, (ii)~$\E\brk{\nrm{g-\nabla f(y)}^2\mid y}\leq\sigma^2$, and (iii)~$\E\brk{\nrm{g}^2\mid y}\leq G^2$.
\end{assumption}

Since $f$ is locally Lipschitz, its differentiability implies that $f$ is \emph{well-behaved}~\citep[Proposition~2]{cutkosky2023}, that is,
$f(y)-f(x)=\int_0^1\inpr{\nabla f\bigl(x+t(y-x)\bigr),\,y-x}\,\dd t$ for all $x,y\in\R^d$.

When $f$ is not differentiable, we replace it by
$f_\rho(x)\coloneqq\E_{z\sim\unif(\mathbb{B}(0,1))}\brk*{f(x+\rho z)}$ with a
smoothing radius $\rho>0$ and carry out the analysis on $f_\rho+h$.
Then $f_\rho$ is differentiable with
$\nabla f_\rho(x)=\E_{z\sim\unif(\mathbb{B}(0,1))}\brk*{\nabla f(x+\rho z)}$, where $f$ is differentiable at $x+\rho z$ almost surely, and a stochastic gradient of $f$ at $x+\rho z$ is a stochastic gradient of $f_\rho$ at $x$ with the same second-moment bound~\citep[Proposition~2]{cutkosky2023}.
Moreover, if $f$ is $L$-Lipschitz, it holds that
$\abs{f_\rho(x)-f(x)}\leq\rho L$ for every $x\in\R^d$.
Stationarity guarantees for $f_\rho+h$ are transferred to the objective $f+h$ by \cref{lem:smoothing-transfer}, with the Goldstein radius enlarged by $\rho$.

For zeroth-order access, we replace \cref{ass:oracle} with the following assumption on the sample functions, following \citet{kornowski2024zo}.

\begin{assumption}\label{ass:zo-estimator}
For any $\xi$, the function $F(\cdot,\xi)$ is
$L(\xi)$-Lipschitz, where $\E_\xi\brk*{L(\xi)^2}\leq L_0^2$.
\end{assumption}

\subsection{Goldstein-type stationarity}

Here, we first define the Goldstein subdifferential and then
use it to define proximal Goldstein stationary points for the composite
objective $f+h$, following \citet{chen2025}.  For a locally Lipschitz $f$,
the Clarke generalized directional derivative is
$f^\circ(x;v)\coloneqq\limsup_{y\to x,\,t\downarrow0}
\tfrac{f(y+tv)-f(y)}{t}$ and the Clarke subdifferential is
$\partial f(x)\coloneqq\set*{g\in\R^d \colon \inpr{g,v}\leq f^\circ(x;v)
 \, \text{for all} \, v\in\R^d}$~\citep{clarke1990}.
Then, for $\delta\geq0$, define the Goldstein $\delta$-subdifferential~\citep{goldstein1977} by
\begin{align}
  \partial_\delta f(x)
  \coloneqq\operatorname{conv}\set[\Bigg]{\bigcup_{y\in\mathbb{B}(x,\delta)}\partial f(y)}.
  \label{eq:goldstein-subdifferential}
\end{align}
We next define the proximal operator and the proximal-gradient mapping~\citep{nesterov2013composite,parikh2014proximal,ghadimi2016minibatch}.
For $\gamma>0$, $z\in\R^d$, $x\in\operatorname{dom}h$, and $g\in\R^d$, the
proximal operator of $h$, $\operatorname{prox}_{\gamma h}(z)$, is defined as
\begin{equation}
  \operatorname{prox}_{\gamma h}(z)
  \coloneqq
  \arg\min_{y\in\R^d}
  \set*{h(y)+\frac{1}{2\gamma}\nrm{y-z}^2},
  \label{eq:prox-definition}
\end{equation}
and the proximal-gradient mapping $G_{\gamma h}(x,g)$ as
\begin{equation}
  G_{\gamma h}(x,g)
  \coloneqq
  \frac{1}{\gamma}
  \prn*{x-\operatorname{prox}_{\gamma h}(x-\gamma g)}.
  \label{eq:prox-gradient-map}
\end{equation}
Note that the minimizer in \cref{eq:prox-definition} is unique since $h(y)+\frac{1}{2\gamma}\nrm{y-z}^2$ is strongly convex.  

We use the following extension of Goldstein stationarity to composite objectives $f+h$~\citep{chen2025}.
\begin{definition}[{PGSP, \citealt{chen2025}}]\label{def:pgsp}
A point $x$ is said to be a $(\gamma,\delta,\epsilon)$-proximal Goldstein stationary
point (PGSP) if
\begin{equation}
  \min_{g\in\partial_\delta f(x)}\nrm{G_{\gamma h}(x,g)}\leq\epsilon.
  \label{eq:PSGP}
\end{equation}
\end{definition}
This relates to other stationarity conditions as follows:
\begin{enumerate}[label=\textup{(\roman*)},labelindent=2pt,itemsep=0pt,topsep=0pt,parsep=0pt]
\item 
In (noncomposite) nonsmooth nonconvex optimization, $h=0$ gives $G_{\gamma h}(x,g)=g$.
Thus, \cref{eq:PSGP} is the $(\delta,\epsilon)$-Goldstein stationarity condition $\operatorname{dist}\prn*{0,\partial_\delta f(x)}\leq\epsilon$, independently of $\gamma$~\citep{zhang2020goldstein}.

\item 
In smooth nonconvex optimization, where $\nabla f$ is $L_1$-Lipschitz continuous, we have $\partial f(x)=\set*{\nabla f(x)}$.
In this case, $\nrm{G_{\gamma h}(x,\nabla f(x))}$ for a fixed $\gamma>0$ is the standard stationarity measure~\citep{nesterov2013composite,ghadimi2016minibatch}, and we call $x\in\operatorname{dom}h$ an $\epsilon$-stationary point of $\Phi$ if $\nrm{G_{\gamma h}(x,\nabla f(x))}\leq\epsilon$.
As shown by \citet[Proposition~2]{chen2025}, a $(\gamma,\epsilon/(2L_1),\epsilon/2)$-PGSP is an $\epsilon$-stationary point of $\Phi$, and 
conversely, an $\epsilon$-stationary point of $\Phi$ is a
$(\gamma,\delta,\epsilon)$-PGSP for every $\delta\geq0$.

\item 
In constrained nonsmooth nonconvex optimization, where $h$ in \cref{eq:intro-problem} is zero on $\Omega$ and $+\infty$ outside it,
the operator $\operatorname{prox}_{\gamma h}$ becomes the Euclidean projection onto
$\Omega$, and \cref{def:pgsp} coincides with the
$(\gamma,\delta,\epsilon)$-generalized Goldstein stationarity condition
of \citet[Definition~4.2]{liu2024constrained}.
\item
In weakly convex optimization, stationarity is commonly measured by the gradient norm of the Moreau envelope, which also relates to PGSP.
For the smoothed composite objective $\Phi_\delta=f_\delta+h$, let $L_\delta$ be a Lipschitz constant of $\nabla f_\delta$ and choose $0<\gamma<1/L_\delta$.
A point $x$ satisfying $\nrm{\nabla e_\gamma\Phi_\delta(x)}\leq\epsilon$ is a $(\gamma,\delta,(1+\gamma L_\delta)\epsilon)$-PGSP, where $e_\gamma\Phi_\delta$ denotes the Moreau envelope of $\Phi_\delta$ (see \cref{prop:moreau-pgsp} in \cref{app:moreau-pgsp}).
\end{enumerate}

For random outputs, stationarity guarantees bound the expected value of the corresponding residual.

The following lemma transfers a PGSP guarantee from the smoothed objective $f_\rho+h$ to $f+h$ by increasing the Goldstein radius by $\rho$, and directly follows from~\citet[Lemma~4]{kornowski2024zo}.

\begin{lemma}\label{lem:smoothing-transfer}
For any $\rho,\nu\geq0$ and any $x$,
$\partial_\nu f_\rho(x)\subseteq\partial_{\rho+\nu}f(x)$.
Consequently, for any $\gamma>0$, if $x$ is a $(\gamma,\nu,\epsilon)$-PGSP for $f_\rho+h$, then
it is a $(\gamma,\rho+\nu,\epsilon)$-PGSP for $f+h$.
\end{lemma}

This lemma lets us analyze the differentiable
function $f_\rho$ even when $f$ is nonsmooth and transfer the
resulting guarantee to $f+h$ at any target $\delta\geq\rho+\nu$.

%% file: section/04_proposed_algorithm.tex
\section{Composite O2NC}\label{sec:algorithm}

\begin{algorithm}[t]
  \caption{\CompositeONC{}}\label{alg:pgsp}
  \DontPrintSemicolon
  \Require{initial point $x_1\in\operatorname{dom}h$, radius $D>0$, block length $T$, number of blocks $K$, query oracle $\mathsf Q$, online learner.}
  \For{$k=1,\ldots,K$}{
    Initialize the online learner.\nllabel{algline:pgsp-init}\;
    \For{$t=1,\ldots,T$}{
      Online learner outputs $u_{k,t}\in\calU_D(x_k)$.\nllabel{algline:pgsp-displacement}\;
      Draw $s_{k,t}\sim\unif[0,1]$ and set $y_{k,t}=x_k+s_{k,t}u_{k,t}$.\nllabel{algline:pgsp-sample}\;
      Obtain $g_{k,t}$ using the available oracle:\nllabel{algline:pgsp-query}\;
      (i) First-order access: $g_{k,t}=\mathsf Q_{\mathrm{fo}}(y_{k,t})$, satisfying \cref{ass:oracle}.\;
      (ii) Zeroth-order access: $g_{k,t}=\mathsf Q_{\mathrm{zo}}(y_{k,t})$ 
      $=\hat g_\rho(y_{k,t};\xi,w)$, defined in \cref{eq:two-point-estimator}.\;
      Give the learner the convex loss in \cref{eq:new_loss}: $\ell_{k,t}(u)=\inpr{g_{k,t},u}+h(x_k+u)-h(x_k)$.\nllabel{algline:pgsp-loss}\;
    }
    Define $\bar g_k=T^{-1}\sum_{t=1}^Tg_{k,t}$.\nllabel{algline:pgsp-average}\;
    Draw $I_k\sim\unif\set{1,\ldots,T}$ independently and set $x_{k+1}=x_k+u_{k,I_k}$.\nllabel{algline:pgsp-next-anchor}\;
  }
  Draw $R\sim\unif\set{1,\ldots,K}$ independently.\nllabel{algline:pgsp-select}\;
  Shift the selected anchor by $u_R^*\in\argmin_{u\in\calU_D(x_R)}\set{\inpr{\bar g_R,u}+h(x_R+u)-h(x_R)}$.\nllabel{algline:pgsp-shift}\;
  \Return{the shifted point $y_R=x_R+u_R^*$.}
\end{algorithm}

\begin{figure}[t]
  \centering
  \includegraphics[width=0.6\columnwidth]{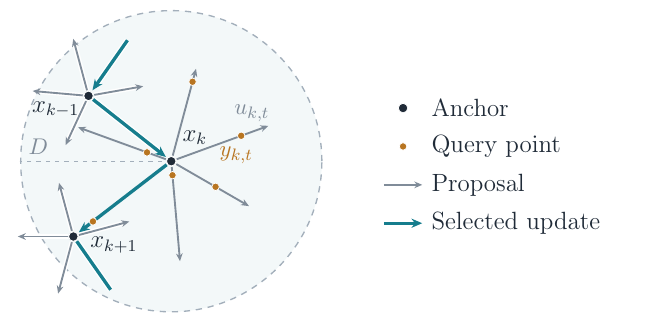}
  \caption{Query points and iterate updates in \CompositeONC{}.
  Each anchor (black dot) is fixed for $T$ rounds.
  The dashed circle denotes $\mathbb B(x_k,D)$.}
  \label{fig:anchor-exploration}
\end{figure}

This section establishes our conversion framework, composite online-to-nonconvex conversion (\CompositeONC{}), and provides the oracle complexity guarantees for the first-order and zeroth-order setting.

\subsection{Proposed conversion framework}
Here we provide \CompositeONC{}, summarized in \Cref{alg:pgsp} and illustrated in \cref{fig:anchor-exploration}.
The versions with first-order access and zeroth-order access differ only
in the query oracle $\mathsf Q$ used in \cref{algline:pgsp-query}.
In what follows, we describe \Cref{alg:pgsp} in detail.

The algorithm runs $K$ blocks of $T$ rounds, and in each block $k \in [K]$ the update is centered at an
anchor $x_k$.  At each anchor, an online learner chooses a direction $u_{k,t}$ from $\calU_D(x_k)\coloneqq\set*{u\in\R^d\colon\nrm{u}\leq D,\ x_k+u\in\operatorname{dom}h}$ in each round $t$ (\cref{algline:pgsp-displacement}).
The algorithm then draws $s_{k,t}\sim\unif[0,1]$, sets
$y_{k,t}=x_k+s_{k,t}u_{k,t}$, and queries
$g_{k,t}=\mathsf Q(y_{k,t})$ (\cref{algline:pgsp-sample,algline:pgsp-query}).
With first-order access, $\mathsf Q$ returns an unbiased stochastic
gradient of $f$ at $y_{k,t}$, as in \cref{ass:oracle}.  With zeroth-order
access, $\mathsf Q$ uses function values at two nearby points to construct an unbiased estimator of
$\nabla f_\rho(y_{k,t})$, as defined in \cref{eq:two-point-estimator}.
The algorithm then gives the learner the loss
\begin{equation}
\ell_{k,t}(u)=\inpr{g_{k,t},u}+h(x_k+u)-h(x_k),
\label{eq:new_loss}
\end{equation}
as in \cref{algline:pgsp-loss}.  This loss is convex in $u$ since $h$ is convex.

Let $\calF_{k,t}$ denote the history before $s_{k,t}$ is drawn, which determines $u_{k,t}$.
Suppose that $\E\brk*{g_{k,t}\mid y_{k,t}}=\nabla f(y_{k,t})$, which holds under first-order access by \cref{ass:oracle} and under zeroth-order access with $f_\rho$ in place of $f$ (\cref{sec:zo-transfer}).
Then, the conditional expectation of the loss can be rewritten as
\begin{align}
  \E\brk*{\ell_{k,t}(u_{k,t})\mid\calF_{k,t}} 
  &=
  \int_0^1\inpr{\nabla f(x_k+su_{k,t}),u_{k,t}}\,\dd s
  +h(x_k+u_{k,t})-h(x_k)\\
  &=\Phi(x_k+u_{k,t})-\Phi(x_k),
  \label{eq:segment-loss-identity}
\end{align}
where the first equality follows from the unbiasedness of $g_{k,t}$ and $s_{k,t}\sim\unif[0,1]$, and the second equality from the well-behavedness of $f$.

After $T$ rounds, the algorithm draws $I_k$
independently and uniformly from $\set{1,\ldots,T}$ and sets
$x_{k+1}=x_k+u_{k,I_k}$ (\cref{algline:pgsp-next-anchor}).
Due to this update rule, averaging over $I_k$ and using \cref{eq:segment-loss-identity} gives
\begin{equation}
\E\brk*{\Phi(x_{k+1})-\Phi(x_k)}
=\E\brk*{\frac{1}{T}\sum_{t=1}^T\ell_{k,t}(u_{k,t})}.
\end{equation}
Thus, we can estimate the function difference $\Phi(x_k+u_{k,t})-\Phi(x_k)$ by the cumulative loss based on~\cref{eq:new_loss}, which is connected to the regret as we will see below.

After the $K$ blocks, the algorithm draws $R$ uniformly from $[K]$ (\cref{algline:pgsp-select}) and returns the shifted point $y_R=x_R+u_R^*$, where $u_R^*$ is the minimizer in \cref{algline:pgsp-shift}.  The analyses in \cref{sec:pgsp,sec:zo-transfer} establish PGSP guarantees at the returned point $y_R$.

We measure the online learner's performance in block $k$ by comparing its
cumulative loss with that of the best fixed direction in
$\calU_D(x_k)$.  We use a common upper bound $R_T$ on this expected regret
over all blocks in $[K]$ to state the guarantees below,
\begin{equation}
  \E\brk*{
    \sum_{t=1}^T \ell_{k,t}(u_{k,t})
    -\min_{u\in\calU_D(x_k)}\sum_{t=1}^T \ell_{k,t}(u)
  }\leq R_T.
  \label{eq:pgsp-regret}
\end{equation}

Euclidean composite objective mirror descent (COMID)~\citep[Theorem~2]{duchi2010} achieves the following bound on $R_T$.
\begin{theorem}\label{thm:comid-regret}
Suppose that $\E\brk*{\nrm{g_{k,t}}^2 \mid y_{k,t}}\leq G^2$ for every $k,t$.
Then, COMID achieves $R_T\leq 2DG\sqrt{T}$.
\end{theorem}

To relate the best fixed loss in \cref{eq:pgsp-regret} to PGSP, we define the
local composite gap for $x\in\operatorname{dom}h$, $g\in\R^d$, and $D>0$ by
\begin{align}
  \calC_D(x,g)\coloneqq-\frac{1}{D}\min_{u\in\calU_D(x)}
  \set*{\inpr{g,u}+h(x+u)-h(x)}.
  \label{eq:local-gap}
\end{align}
It is nonnegative since $u=0\in\calU_D(x)$ gives the value $0$.  Let
$\bar g_k\coloneqq T^{-1}\sum_{t=1}^T g_{k,t}$ be the average of the gradient
estimates returned in block $k$, as in \cref{algline:pgsp-average}.
Then
the best fixed average loss is expressed in terms of the local composite gap $\calC_D(x,g)$ as follows:
\begin{align}
  \frac{1}{T}\min_{u\in\calU_D(x_k)}\sum_{t=1}^T\ell_{k,t}(u)
  &=\min_{u\in\calU_D(x_k)}\set*{\inpr{\bar g_k,u}+h(x_k+u)-h(x_k)}\\
  &=-D\calC_D(x_k,\bar g_k).
  \label{eq:regret-comparator-gap}
\end{align}

At \cref{algline:pgsp-shift}, \Cref{alg:pgsp} minimizes the local
composite model around the selected anchor $x_R$ and returns the shifted
point $y_R$.
This shift allows us to obtain the convergence guarantee to the PGSP simultaneously for every $\gamma>0$, as the following lemma shows.
\begin{lemma}\label{lem:shifted-conversion}
Fix $x\in\operatorname{dom}h$, $D>0$, and $\nu\geq2D$.
Let $g\in\partial_D f(x)$ and $\hat g\in\R^d$.
Let $u^* \in \argmin_{u\in\calU_D(x)}\set{ \inpr{\hat g,u}+h(x+u)-h(x) }$ and let $y=x+u^*$.
Then, for every $\gamma>0$,
\begin{equation}
  \min_{v\in\partial_\nu f(y)}\nrm{G_{\gamma h}(y,v)}
  \leq\calC_D(x,\hat g)+\nrm{g-\hat g}.
  \label{eq:shifted-conversion-bound}
\end{equation}
\end{lemma}
As discussed in the introduction, existing approaches~\citep{nesterov2013composite,parikh2014proximal,ghadimi2016minibatch,chen2025} use the parameter of the proximal map as an update step size.  
By contrast, our algorithm does not require tuning $\gamma$, and the complexity for finding a PGSP is independent of $\gamma$.

%% file: section/05_first_order.tex
\subsection{\texorpdfstring{First-order complexity}{First-order complexity}}\label{sec:pgsp}

We first analyze \Cref{alg:pgsp} with first-order access to $f$,
setting $\mathsf Q=\mathsf Q_{\mathrm{fo}}$, where
$\mathsf Q_{\mathrm{fo}}(y)$ calls the stochastic gradient oracle of
\cref{ass:oracle} at $y$.

Let $\widetilde g_k\coloneqq T^{-1}\sum_{t=1}^T\nabla f(y_{k,t})$ be the average
gradient in block $k$.  Since all query points lie in $\mathbb B(x_k,D)$,
we have $\widetilde g_k\in\partial_D f(x_k)$.
By \cref{lem:shifted-conversion}, it suffices to upper bound the local composite gap $\calC_D(x_R,\bar g_R)$ and the gradient estimation error $\nrm{\bar g_R-\widetilde g_R}$.
The following lemma gives these bounds.
\begin{lemma}\label{lem:regret-gap-noise}
Suppose that $f$ is well-behaved and that the oracle satisfies
conditions (i) and (ii) of \cref{ass:oracle}.
If an online learner satisfies \cref{eq:pgsp-regret}, then
\begin{equation}
  \E\brk*{\calC_D(x_R,\bar g_R)}\leq\frac{\Delta}{DK}+\frac{R_T}{DT},\quad \E\brk*{\nrm{\bar g_R-\widetilde g_R}}\leq\frac{\sigma}{\sqrt T}.
\end{equation}
\end{lemma}
\begin{proof}
By \cref{eq:pgsp-regret,eq:regret-comparator-gap}, we have
\begin{align}
  D\E\brk*{\calC_D(x_k,\bar g_k)}
  &\leq
  \frac{R_T}{T}
  -
  \E\brk*{\frac{1}{T}\sum_{t=1}^T\ell_{k,t}(u_{k,t})}=
  \frac{R_T}{T}+\E\brk*{\Phi(x_k)-\Phi(x_{k+1})},
\end{align}
where the last inequality uses \cref{eq:segment-loss-identity}.
Summing over $k\in[K]$ and using $R\sim\unif\set{1,\ldots,K}$ gives
\begin{align}
  \E\brk*{\calC_D(x_R,\bar g_R)}
  &\leq
  \frac{R_T}{DT}
  +
  \frac{\Phi(x_1)-\E\brk*{\Phi(x_{K+1})}}{DK}
  \leq
  \frac{R_T}{DT}
  +
  \frac{\Delta}{DK},
\end{align}
which completes the proof of the first statement.

For the second statement, by Jensen's inequality,
\begin{align}
  \E\brk*{\nrm{\bar g_R-\widetilde g_R}}
  \leq\sqrt{\E\brk*{\nrm{\bar g_R-\widetilde g_R}^2}}
  =\sqrt{
    \frac{1}{KT^2}\sum_{k=1}^K\sum_{t=1}^T
    \E\brk*{\nrm{g_{k,t}-\nabla f(y_{k,t})}^2}}
  \leq\frac{\sigma}{\sqrt T},
\end{align}
where the equality follows from $R \sim \mathsf{Unif}\set{1,\dots,K}$ and the conditional unbiasedness of $g_{k,t}$ for $\nabla f(y_{k,t})$ removes the cross terms,
and the last inequality follows from the variance bound in \cref{ass:oracle}.
\end{proof}

Combining \cref{lem:shifted-conversion,lem:regret-gap-noise} with the COMID
regret bound in \cref{thm:comid-regret} gives the following stochastic gradient query complexity.
\begin{theorem}\label{thm:first-order-complexity}
Suppose that \cref{ass:oracle} holds, fix $\delta>0$ and
$0<\epsilon\leq\min\set*{G+\sigma,\Delta/\delta}$, and let $D=\delta/2$,
$T=\max\set*{1,\ceil*{4(2G+\sigma)^2/\epsilon^2}}$, and
$K=\max\set*{1,\ceil*{2\Delta/(D\epsilon)}}$.
Then \Cref{alg:pgsp} with the COMID learner of \cref{thm:comid-regret} and $\mathsf Q=\mathsf Q_{\mathrm{fo}}$ outputs $y_R$ that is a $(\gamma,\delta,\epsilon)$-PGSP in expectation
for every $\gamma>0$, using
$O\prn[\Big]{\frac{\Delta(G+\sigma)^2}
  {\delta\epsilon^3}}$ stochastic gradient queries.
\end{theorem}

This bound matches the $O(\delta^{-1}\epsilon^{-3})$ complexity of \citet{cutkosky2023} for $h=0$.
Since $h=0$ is a special case of our setting, their joint lower bound $\Omega(\delta^{-1}\epsilon^{-3})$ also applies to our composite problem class, establishing optimal dependence on $\delta$ and $\epsilon$.

\begin{proof}
Since $\widetilde g_R\in\partial_D f(x_R)$ and $\delta=2D$,
using \cref{lem:shifted-conversion} with $x=x_R$, $g=\widetilde g_R$, and $\hat g=\bar g_R$ gives, for every
$\gamma>0$,
\begin{align}
  \E\brk*{\min_{g\in\partial_\delta f(y_R)}\nrm{G_{\gamma h}(y_R,g)}} 
  &\leq 
  \E\brk*{\calC_D(x_R,\bar g_R)+\nrm{\bar g_R-\widetilde g_R}}
  \\
  &\leq\frac{\Delta}{DK}+\frac{R_T}{DT}+\frac{\sigma}{\sqrt T}\\
  &\leq\frac{\Delta}{DK}+\frac{2G+\sigma}{\sqrt T}
  \leq\frac{\epsilon}{2}+\frac{\epsilon}{2}=\epsilon,
\end{align}
where the second inequality follows from \cref{lem:regret-gap-noise} and the third inequality uses $R_T\leq2DG\sqrt T$ from
\cref{thm:comid-regret}.
Since $\epsilon\leq\Delta/\delta$ and $\epsilon\leq G+\sigma$,
the parameter choices give $K=O\prn*{\Delta/(\delta\epsilon)}$
and $T=O\prn*{(G+\sigma)^2/\epsilon^2}$.
Their product $KT$ gives the stated number of stochastic gradient queries, which completes the proof.
\end{proof}

%% file: section/06_zeroth_order.tex
\subsection{\texorpdfstring{Zeroth-order complexity}{Zeroth-order complexity}}\label{sec:zo-transfer}
Here, we analyze \Cref{alg:pgsp} with zeroth-order access under \cref{ass:zo-estimator}.
Following \citet{kornowski2024zo}, we apply the algorithm to $\Phi_\rho\coloneqq f_\rho+h$, where $f_\rho$ is the uniformly smoothed function in \cref{sec:setting} with radius $\rho>0$.  Write
$\Delta_\rho\coloneqq\Phi_\rho(x_1)-\inf_x\Phi_\rho(x)$ for its initial gap.

At each query point $y$, draw independent samples
$\xi$ and $w\sim\unif(\mathbb{S}^{d-1})$, and set
$\mathsf Q_{\mathrm{zo}}(y)\coloneqq\hat g_\rho(y;\xi,w)$, where
\begin{equation}
  \hat g_\rho(y;\xi,w)\coloneqq\frac{d}{2\rho}\prn*{F(y+\rho w,\xi)-F(y-\rho w,\xi)}w.
  \label{eq:two-point-estimator}
\end{equation}
Each estimator call uses two function-value queries.
The following lemma gives the mean and a second-moment bound for this estimator.

\begin{lemma}\label{lem:two-point-estimator}
Under \cref{ass:zo-estimator}, set
$G_{\mathrm{zo}}\coloneqq4(2\pi)^{1/4}L_0\sqrt{d}$.
Then $f$ is $L_0$-Lipschitz, and for
every $y \in \R^d$, the estimator \cref{eq:two-point-estimator} satisfies
\begin{align}
  &\E_{\xi, w}\brk*{\hat g_\rho(y;\xi,w)\mid y}
  =\nabla f_\rho(y),
  \label{eq:two-point-moment1}\\
  &\E_{\xi, w}\brk*{\nrm{\hat g_\rho(y;\xi,w)}^2\mid y}
  \leq G_{\mathrm{zo}}^2.
  \label{eq:two-point-moment2}
\end{align}
Consequently, $\mathsf Q_{\mathrm{zo}}$ satisfies \cref{ass:oracle}
for $f_\rho$ with $G=\sigma=G_{\mathrm{zo}}$.
\end{lemma}

Using the argument of \cref{thm:first-order-complexity} for $\Phi_\rho$ and
\cref{lem:smoothing-transfer} gives the guarantee for $f+h$ stated below.

\begin{theorem}\label{thm:zeroth-order-complexity}
Suppose that \cref{ass:zo-estimator} holds, fix $\delta>0$ and
$0<\epsilon\leq L_0$, and let $D=\rho=\delta/3$,
$T=\max\set*{1,\ceil*{576\sqrt{2\pi}dL_0^2/\epsilon^2}}$, and
$K=\max\set*{1,\ceil*{2(\Delta+\delta L_0)/(D\epsilon)}}$.
Then \Cref{alg:pgsp} with the COMID learner of \cref{thm:comid-regret} and $\mathsf Q=\mathsf Q_{\mathrm{zo}}$ outputs $y_R$ that is a $(\gamma,\delta,\epsilon)$-PGSP in expectation
for every $\gamma>0$, using
$O\prn*{\frac{dL_0^2(\Delta+\delta L_0)}
  {\delta\epsilon^3}}$ function-value queries.
\end{theorem}

This bound improves on the $O(d^{3/2}\delta^{-1}\epsilon^{-4})$ and $O(d^{3/2}\delta^{-1}\epsilon^{-3})$ complexities obtained by \citet{chen2025} with minibatch estimation and variance reduction, respectively.
It also improves on the $O(d^{3/2}\delta^{-1}\epsilon^{-4})$ function-value query complexity for PGSPs obtained from \citet{pougkakiotis2026inexact} under the oracle conditions detailed in \cref{app:moreau-pgsp}.
It matches the $O(d\delta^{-1}\epsilon^{-3})$ complexity of \citet{kornowski2024zo} for $h=0$ and is optimal in each of $d$, $\delta$, and $\epsilon$ individually.

%% file: section/07_smooth_setting.tex
\section{Smooth setting}\label{sec:smooth}

This section derives oracle complexity bounds for finding an
$\epsilon$-stationary point of $\Phi$ when $f$ is smooth.
In particular, this section assumes that $f$ is $L_1$-smooth, that is,
$\nrm{\nabla f(x)-\nabla f(y)}\leq L_1\nrm{x-y}$ for all $x,y\in\R^d$.
The proofs of this section are deferred to \cref{app:smooth-proofs}.

The proximal-gradient mapping $G_{\gamma h}(x,\nabla f(x))$ is commonly
used to define $\epsilon$-stationarity for smooth composite objectives~\citep{nesterov2013composite,ghadimi2016minibatch}, as in
\cref{sec:preliminaries}.
The following lemma relates $\epsilon$-stationarity to the PGSP condition.

\begin{lemma}\label{lem:smooth-stationarity}
When $f$ is $L_1$-smooth, for every
$x\in\operatorname{dom}h$, $\gamma>0$, and $\delta\geq0$, it holds that
$\nrm{G_{\gamma h}(x,\nabla f(x))} \leq\min_{g\in\partial_\delta f(x)}\nrm{G_{\gamma h}(x,g)}+L_1\delta.$
\end{lemma}

Combining \cref{lem:smooth-stationarity,thm:first-order-complexity} with
$\delta=\epsilon/(2L_1)$ gives $O(\epsilon^{-4})$ stochastic gradient
query complexity for finding an $\epsilon$-stationary point of $\Phi$.

To obtain a sharper rate in the deterministic case $\sigma=0$, we
incorporate the idea of optimistic online mirror descent~\citep{chiang2012gradual,rakhlin2013predictable} into COMID (see \cref{app:proofs-comid} for the update and the proof).
The optimistic COMID learner uses the previous stochastic gradient as a hint.
To initialize the hint in each block, let $g_{k,0}$ be an independent
stochastic gradient queried at $x_k$.
Since the query points in each block lie in $\mathbb B(x_k,D)$,
smoothness and the variance bound give $\E\brk*{\nrm{g_{k,t}-g_{k,t-1}}^2}\leq12(L_1D+\sigma)^2$.
Using this bound in the regret analysis yields the following guarantee
without the second-moment bound in condition (iii) of \cref{ass:oracle}.

\begin{theorem}\label{thm:optimistic-comid-regret}
Suppose that $f$ is $L_1$-smooth and that conditions (i) and (ii) of \cref{ass:oracle} hold. Then, optimistic COMID achieves
$R_T\leq2\sqrt3\,(L_1D+\sigma)D\sqrt T$.
\end{theorem}
Note that optimistic COMID is not limited to the smooth setting: the guarantee of \cref{thm:first-order-complexity} also holds with this learner in the nonsmooth setting.

We now derive smooth counterparts of
\cref{thm:first-order-complexity,thm:zeroth-order-complexity} for first-order
and zeroth-order access.

\begin{theorem}\label{thm:smooth-complexity}
Suppose that $f$ is $L_1$-smooth and that the oracle satisfies
conditions (i) and (ii) of \cref{ass:oracle}, fix $\epsilon>0$, and let
$D=\epsilon/(12L_1)$,
$T=\max\set*{1,\ceil*{400\sigma^2/\epsilon^2}}$, and
$K=\max\set*{1,\ceil*{48L_1\Delta/\epsilon^2}}$.
Then \cref{alg:pgsp} with the optimistic COMID learner of \cref{thm:optimistic-comid-regret} and $\mathsf Q=\mathsf Q_{\mathrm{fo}}$ outputs a point $y_R$ that is an $\epsilon$-stationary point of $\Phi$
for every $\gamma>0$, that is,
$\E\brk*{\nrm{G_{\gamma h}(y_R,\nabla f(y_R))}}\leq\epsilon$,
using
$O\prn*{\max\set*{
  \frac{L_1\Delta}{\epsilon^2},\frac{\sigma^2}{\epsilon^2},
  \frac{L_1\Delta\sigma^2}{\epsilon^4}}}$ stochastic gradient queries.
\end{theorem}

In the deterministic case $\sigma=0$, \cref{thm:smooth-complexity} uses
$O(L_1\Delta/\epsilon^2)$ gradient queries to find an
$\epsilon$-stationary point of $\Phi$.
The complexity in \cref{thm:smooth-complexity} matches that of proximal
stochastic gradient methods~\citep{ghadimi2016minibatch}.
Its $\sigma^2\epsilon^{-4}$ dependence also matches the lower bound for
unbiased bounded-variance oracles, even when $h=0$~\citep{arjevani2023lower}.
\Cref{tab:smooth-comparison} in \cref{app:smooth-comparison} summarizes these comparisons.

\begin{theorem}\label{thm:smooth-zo-complexity}
Suppose that \cref{ass:zo-estimator} holds and that $f$ is $L_1$-smooth, fix $\epsilon>0$, and let
$\rho=\epsilon/(2L_1)$, $D=\epsilon/(24L_1)$,
$T=\max\set*{1,\ceil*{1600G_{\mathrm{zo}}^2/\epsilon^2}}$, and
$K=\ceil*{192L_1\Delta/\epsilon^2+48}$,
where $G_{\mathrm{zo}}=4(2\pi)^{1/4}L_0\sqrt d$ as in
\cref{lem:two-point-estimator}.
Then \cref{alg:pgsp} with the optimistic COMID learner of \cref{thm:optimistic-comid-regret} and $\mathsf Q=\mathsf Q_{\mathrm{zo}}$ outputs a point $y_R$ that is an $\epsilon$-stationary point of $\Phi$
for every $\gamma>0$, that is,
$\E\brk*{\nrm{G_{\gamma h}(y_R,\nabla f(y_R))}}\leq\epsilon$,
using
$O\prn*{\max\set*{
  \frac{L_1\Delta}{\epsilon^2},\frac{dL_0^2}{\epsilon^2},
  \frac{dL_0^2L_1\Delta}{\epsilon^4}}}$ function-value queries.
\end{theorem}

The leading term in \cref{thm:smooth-zo-complexity} is
$O(dL_0^2L_1\Delta\epsilon^{-4})$ for finding an
$\epsilon$-stationary point of $\Phi$.
This has the same $d\epsilon^{-4}$ dependence as the bound for smooth
composite objectives in \citet[Corollary~8]{ghadimi2016minibatch}.

%% file: section/08_numerical_figures.tex
\begin{figure}[!t]
  \centering
  \includegraphics[width=0.4\linewidth]{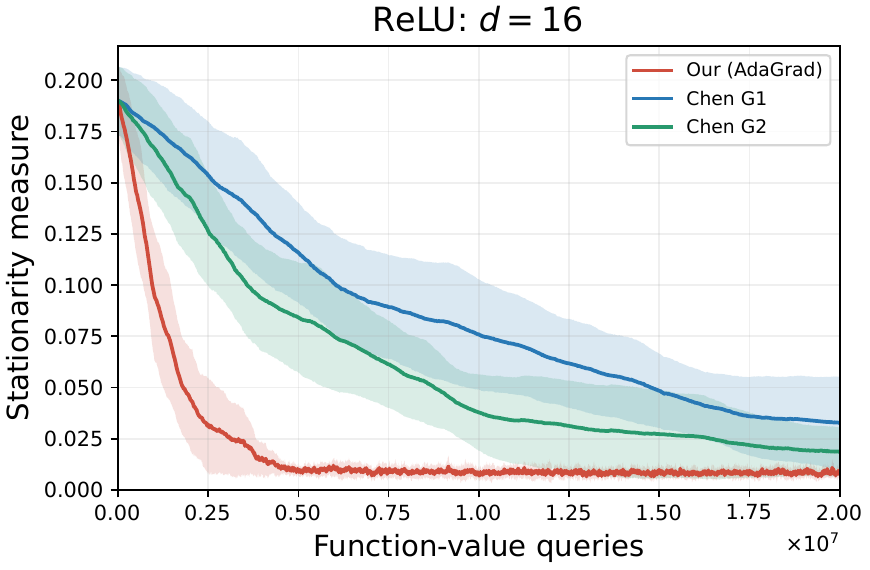}
  \includegraphics[width=0.4\linewidth]{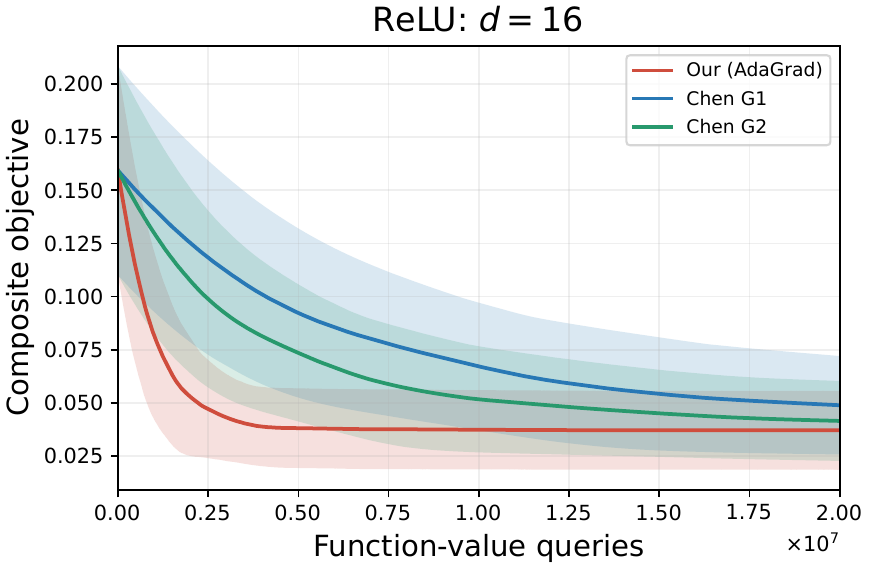}\\
  \includegraphics[width=0.4\linewidth]{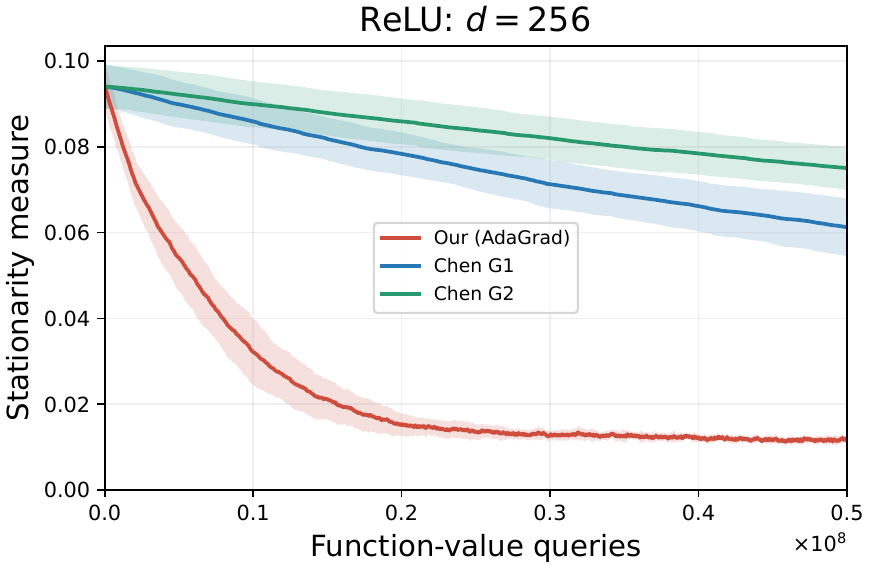}
  \includegraphics[width=0.4\linewidth]{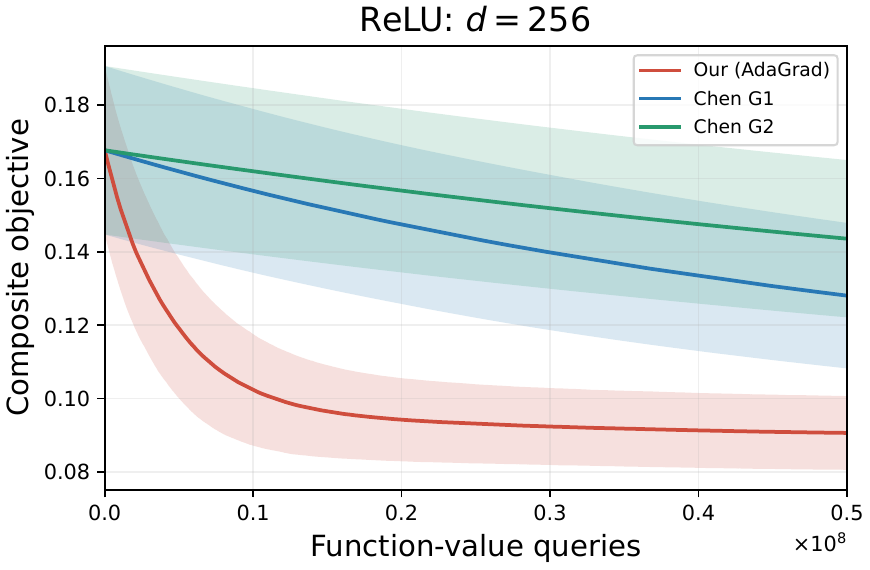}\\
  \includegraphics[width=0.4\linewidth]{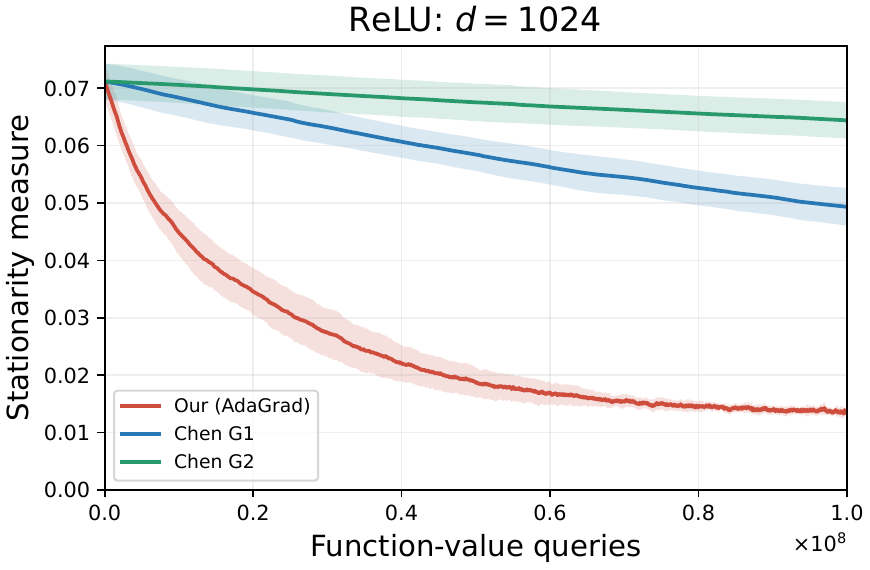}
  \includegraphics[width=0.4\linewidth]{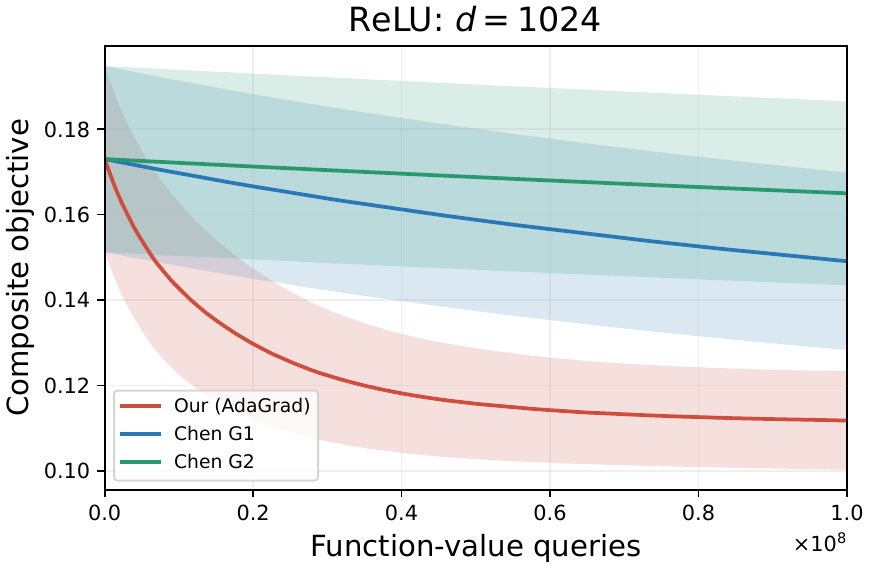}
  \caption{Stationarity measure $\min_{g\in\partial_\delta f(x)}\nrm{G_{\gamma h}(x,g)}$ (left) and objective values (right) for absolute loss for ReLU regression
  at $d=16$ (top), $d=256$ (middle), and $d=1024$ (bottom).
  Curves and shaded bands show means and one standard deviation over ten trials.}
  \label{fig:experiments}
\end{figure}

%% file: section/08_numerical_experiments.tex
\section{Numerical experiments}\label{sec:experiments}

To validate our theoretical results, we evaluate the query efficiency of \CompositeONC{} on nonsmooth nonconvex composite problems.
We conduct zeroth-order experiments on absolute loss for ReLU regression at $d\in\set{16,256,1024}$, with $\ell_1$ and $\ell_2$ regularization and box constraints.
We compare \CompositeONC{} with the minibatch (option G1) and variance-reduced (option G2) methods of \citet{chen2025}, using the same data and initialization for all methods within each trial.
The objective and all experimental parameters are in \cref{app:experimental-details}.

\Cref{fig:experiments} plots the stationarity measure and composite objective value against the number of function-value queries. We evaluate stationarity using a numerical upper bound on the measure with $\delta=0.01$ and $\gamma=1$.
Across all three dimensions, \CompositeONC{} outperforms the other methods in terms of both the stationarity measure and the composite objective value. Additional results for ramp loss are given in \cref{app:experimental-details}.

%% file: section/ai_use_statement.tex
\section*{AI use statement}
We used GPT-5.5, GPT-5.6-Sol, GPT-6.1-Sol, and Claude Opus 5.5 to assist with proof development and checking, literature searches, and editing.
We reviewed all AI-assisted work and take responsibility for the paper.

%% file: section/appendix.tex
\input{section/appendix_comid}
\input{section/appendix_zeroth_order}
\input{section/appendix_smooth}
\input{section/appendix_moreau_pgsp}
\input{section/appendix_experimental_details}

%% file: section/appendix_comid.tex
\section{Deferred proofs of the regret bounds for COMID}\label{app:proofs-comid}

We describe how composite objective mirror descent (COMID) and
optimistic COMID generate the directions $u_{k,t}$ using
\cref{eq:comid-loss-update,eq:optimistic-comid}, respectively, and prove
their regret bounds.

\begin{proof}[Proof of \cref{thm:comid-regret}]
Fix a block $k$ and apply \citet{duchi2010}, in their notation, with
\begin{align}
  \Omega=\calU_D(x_k),
  \qquad f_{k,t}(u)=\inpr{g_{k,t},u},
  \qquad r_k(u)=h(x_k+u)-\min_{v\in\calU_D(x_k)}h(x_k+v).
  \label{eq:comid-identification}
\end{align}
Initialize $u_{k,1}\in\argmin_{v\in\calU_D(x_k)}h(x_k+v)$, so that $r_k(u)\geq0$ for all $u$ and $r_k(u_{k,1})=0$.
Using the Euclidean Bregman divergence, the COMID update~\citep[Eq.~(3)]{duchi2010} specializes to
\begin{align}
  u_{k,t+1}\in\argmin_{u\in\calU_D(x_k)}
  \set*{
    \frac{1}{2}\nrm{u-u_{k,t}}^2
    +\eta\prn*{\inpr{g_{k,t},u} + h(x_k+u)-h(x_k)}
  },
  \label{eq:comid-update}
\end{align}
where $\eta>0$ is the step size.
Then, by the definition of $\ell_{k,t}$ in \cref{eq:new_loss},
\begin{align}
  u_{k,t+1}\in\argmin_{u\in\calU_D(x_k)}
  \set*{\frac{1}{2}\nrm{u-u_{k,t}}^2+\eta\ell_{k,t}(u)}.
  \label{eq:comid-loss-update}
\end{align}

For any $u\in\calU_D(x_k)$, Theorem~2 of \citet{duchi2010} now gives
\begin{align}
  \sum_{t=1}^T \ell_{k,t}(u_{k,t})
    -\sum_{t=1}^T \ell_{k,t}(u)
  &=
  \sum_{t=1}^T\prn*{f_{k,t}(u_{k,t})+r_k(u_{k,t})-f_{k,t}(u)-r_k(u)}\\
  &\leq
  \frac{\nrm{u-u_{k,1}}^2}{2\eta}+r_k(u_{k,1})
  +\frac{\eta}{2}\sum_{t=1}^T\nrm{g_{k,t}}^2.
  \label{eq:comid-regret-application}
\end{align}
Taking the minimum over $u$, then expectation, and using
$\nrm{u-u_{k,1}}\leq2D$ and the assumed moment bound gives
\begin{align}
  \E\brk*{\sum_{t=1}^T\ell_{k,t}(u_{k,t})
    -\min_{u\in\calU_D(x_k)}\sum_{t=1}^T\ell_{k,t}(u)}
  \leq\frac{2D^2}{\eta}+\frac{\eta TG^2}{2}.
\end{align}
Choosing $\eta=2D/(G\sqrt{T})$ makes the right-hand side
$2DG\sqrt{T}$, which completes the proof.
\end{proof}

\begin{proof}[Proof of \cref{thm:optimistic-comid-regret}]
Fix a block $k$ and set $r_k(u)\coloneqq h(x_k+u)-h(x_k)$.
Let $g_{k,0}$ be an independent stochastic gradient queried at $x_k$.
For a step size $\eta>0$, initialize $v_{k,1}=0$ and
\begin{align}
  u_{k,1}\in\argmin_{u\in\calU_D(x_k)}
  \set*{\frac12\nrm{u}^2+\eta\prn*{\inpr{g_{k,0},u}+r_k(u)}}.
\end{align}
We use $g_{k,t}$ as the hint vector for round $t+1$ and apply the COMID update
\cref{eq:comid-loss-update} in the proof of \cref{thm:comid-regret}
as follows
\begin{align}
  v_{k,t+1}
  &\in\argmin_{u\in\calU_D(x_k)}
    \set*{\frac12\nrm{u-v_{k,t}}^2+\eta\ell_{k,t}(u)},
    \label{eq:optimistic-comid-v}\\
  u_{k,t+1}
  &\in\argmin_{u\in\calU_D(x_k)}
    \set*{\frac12\nrm{u-v_{k,t+1}}^2
      +\eta\prn*{\inpr{g_{k,t},u}+r_k(u)}}.
  \label{eq:optimistic-comid}
\end{align}
For any $u\in\calU_D(x_k)$, we have
\begin{align}
  \ell_{k,t}(u_{k,t})-\ell_{k,t}(u)
  &=\ell_{k,t}(v_{k,t+1})-\ell_{k,t}(u)
    +\ell_{k,t}(u_{k,t})-\ell_{k,t}(v_{k,t+1})\\
  &=\ell_{k,t}(v_{k,t+1})-\ell_{k,t}(u)
    +\inpr{g_{k,t},u_{k,t}-v_{k,t+1}}+r_k(u_{k,t})-r_k(v_{k,t+1})\\
  &=\ell_{k,t}(v_{k,t+1})-\ell_{k,t}(u)
    +\inpr{g_{k,t-1},u_{k,t}-v_{k,t+1}}\\
  &\quad+r_k(u_{k,t})-r_k(v_{k,t+1})
    +\inpr{g_{k,t}-g_{k,t-1},u_{k,t}-v_{k,t+1}}.
  \label{eq:optimistic-loss-decomposition}
\end{align}
To bound this difference, we derive the optimality conditions.
For any $w\in\calU_D(x_k)$ and $0<\lambda\leq1$, let
$w_\lambda=v_{k,t+1}+\lambda(w-v_{k,t+1})\in\calU_D(x_k)$.
Then, we have
\begin{align}
  0
  &\leq\frac12\nrm{w_\lambda-v_{k,t}}^2
    -\frac12\nrm{v_{k,t+1}-v_{k,t}}^2
    +\eta\prn*{\ell_{k,t}(w_\lambda)-\ell_{k,t}(v_{k,t+1})}\\
  &=\lambda\inpr{v_{k,t+1}-v_{k,t},w-v_{k,t+1}}
    +\frac{\lambda^2}{2}\nrm{w-v_{k,t+1}}^2
    +\eta\prn*{\ell_{k,t}(w_\lambda)-\ell_{k,t}(v_{k,t+1})}\\
  &\leq\lambda\inpr{v_{k,t+1}-v_{k,t},w-v_{k,t+1}}
    +\frac{\lambda^2}{2}\nrm{w-v_{k,t+1}}^2
    +\eta\lambda\prn*{\ell_{k,t}(w)-\ell_{k,t}(v_{k,t+1})},
\end{align}
where the first inequality follows from \cref{eq:optimistic-comid-v}
and the last inequality follows from convexity of $\ell_{k,t}$, which gives
$\ell_{k,t}(w_\lambda)\leq(1-\lambda)\ell_{k,t}(v_{k,t+1})
+\lambda\ell_{k,t}(w)$.
Dividing by $\lambda$ and letting $\lambda\downarrow0$, we obtain
\begin{align}
  0
  &\leq\inpr{v_{k,t+1}-v_{k,t},w-v_{k,t+1}}
    +\eta\prn*{\ell_{k,t}(w)-\ell_{k,t}(v_{k,t+1})}.
    \label{eq:optimistic-v-optimality}
\end{align}
Similarly, applying this argument to \cref{eq:optimistic-comid} gives
\begin{align}
  0&\leq\inpr{u_{k,t}-v_{k,t},w-u_{k,t}}
    +\eta\prn*{\inpr{g_{k,t-1},w-u_{k,t}}+r_k(w)-r_k(u_{k,t})}.
    \label{eq:optimistic-u-optimality}
\end{align}
Thus, \cref{eq:optimistic-v-optimality,eq:optimistic-u-optimality} give
\begin{align}
  \ell_{k,t}(v_{k,t+1})-\ell_{k,t}(u)
  &\leq\frac{\inpr{v_{k,t}-v_{k,t+1},v_{k,t+1}-u}}{\eta}
  =\frac{\nrm{u-v_{k,t}}^2-\nrm{u-v_{k,t+1}}^2
    -\nrm{v_{k,t+1}-v_{k,t}}^2}{2\eta}
\end{align}
and
\begin{align}
  \inpr{g_{k,t-1},u_{k,t}-v_{k,t+1}}
    +r_k(u_{k,t})-r_k(v_{k,t+1})
  &\leq\frac{\inpr{v_{k,t}-u_{k,t},u_{k,t}-v_{k,t+1}}}{\eta}\\
  &=\frac{\nrm{v_{k,t+1}-v_{k,t}}^2
    -\nrm{u_{k,t}-v_{k,t+1}}^2-\nrm{u_{k,t}-v_{k,t}}^2}{2\eta}.
\end{align}
Substituting these bounds into \cref{eq:optimistic-loss-decomposition} gives
\begin{align}
  \ell_{k,t}(u_{k,t})-\ell_{k,t}(u)
  &\leq\frac{\nrm{u-v_{k,t}}^2-\nrm{u-v_{k,t+1}}^2
    -\nrm{v_{k,t+1}-v_{k,t}}^2}{2\eta}\\
  &\qquad+\frac{\nrm{v_{k,t+1}-v_{k,t}}^2
    -\nrm{u_{k,t}-v_{k,t+1}}^2-\nrm{u_{k,t}-v_{k,t}}^2}{2\eta}\\
  &\qquad  +\inpr{g_{k,t}-g_{k,t-1},u_{k,t}-v_{k,t+1}}\\
  &\leq\frac{\nrm{u-v_{k,t}}^2-\nrm{u-v_{k,t+1}}^2}{2\eta}
    -\frac{\nrm{u_{k,t}-v_{k,t+1}}^2+\nrm{u_{k,t}-v_{k,t}}^2}{2\eta}\\
  &\qquad+\frac{\eta}{2}\nrm{g_{k,t}-g_{k,t-1}}^2
    +\frac{\nrm{u_{k,t}-v_{k,t+1}}^2}{2\eta}\tag{by Young's inequality}\\
  &=\frac{\nrm{u-v_{k,t}}^2-\nrm{u-v_{k,t+1}}^2}{2\eta}
    -\frac{\nrm{u_{k,t}-v_{k,t}}^2}{2\eta}
    +\frac{\eta}{2}\nrm{g_{k,t}-g_{k,t-1}}^2\\
  &\leq\frac{\nrm{u-v_{k,t}}^2-\nrm{u-v_{k,t+1}}^2}{2\eta}
    +\frac{\eta}{2}\nrm{g_{k,t}-g_{k,t-1}}^2.
\end{align}
Summing over $t$ and using $\nrm{u}\leq D$ yields
\begin{align}
  \sum_{t=1}^T\ell_{k,t}(u_{k,t})-\sum_{t=1}^T\ell_{k,t}(u)
  \leq\frac{D^2}{2\eta}
    +\frac{\eta}{2}\sum_{t=1}^T\nrm{g_{k,t}-g_{k,t-1}}^2.
  \label{eq:optimistic-regret}
\end{align}

With the convention $y_{k,0}=x_k$, we have for all $t\geq1$
\begin{align}
  &\E\brk*{\nrm{g_{k,t}-g_{k,t-1}}^2}\\
  &=\E\brk*{\nrm{\prn*{g_{k,t}-\nabla f(y_{k,t})}
    +\prn*{\nabla f(y_{k,t})-\nabla f(y_{k,t-1})}
    +\prn*{\nabla f(y_{k,t-1})-g_{k,t-1}}}^2}\\
  &\leq3\E\brk*{\nrm{g_{k,t}-\nabla f(y_{k,t})}^2}
    +3\E\brk*{\nrm{\nabla f(y_{k,t})-\nabla f(y_{k,t-1})}^2}
    +3\E\brk*{\nrm{g_{k,t-1}-\nabla f(y_{k,t-1})}^2}\\
  &\leq6\sigma^2+3L_1^2\E\brk*{\nrm{y_{k,t}-y_{k,t-1}}^2}\\
  &\leq12L_1^2D^2+6\sigma^2
  \leq12(L_1D+\sigma)^2,
  \label{eq:prediction-error}
\end{align}
where the second inequality uses the $L_1$-smoothness of $f$ and condition (ii) of \cref{ass:oracle}.

This bound and \cref{eq:optimistic-regret} give
\begin{align}
  R_T
  &\leq\frac{D^2}{2\eta}
    +\frac{\eta}{2}\sum_{t=1}^T\E\brk*{\nrm{g_{k,t}-g_{k,t-1}}^2}\leq\frac{D^2}{2\eta}+6\eta T(L_1D+\sigma)^2.
\end{align}
Choosing $\eta=D/\prn{2\sqrt{3T}(L_1D+\sigma)}$ then gives
$R_T\leq2\sqrt3\,(L_1D+\sigma)D\sqrt T$.
\end{proof}

\section[Deferred proofs for Composite O2NC]{Deferred proofs for \cref{sec:algorithm}}\label{app:proofs-conversion}\label{app:proofs-zeroth-order}\label{app:proofs-preliminaries}

We prove \cref{lem:shifted-conversion,lem:two-point-estimator,thm:zeroth-order-complexity} in this section.

\begin{lemma}\label{lem:inclusion-to-proximal}
For $x\in\operatorname{dom}h$, $g\in\R^d$, $\gamma>0$, and every
$s\in\partial h(x)$,
$\nrm{G_{\gamma h}(x,g)}\leq \nrm{g+s}$.
\end{lemma}

\begin{proof}[Proof of \cref{lem:inclusion-to-proximal}]
By \cref{eq:prox-definition},
$\operatorname{prox}_{\gamma h}(x-\gamma g)$ minimizes
$h\prn*{v}+\frac{1}{2\gamma}\nrm{v-\prn*{x-\gamma g}}^2$ over
$v\in\R^d$.  Therefore, its first-order optimality condition is
\begin{align}
  0\in
  \partial h\prn*{\operatorname{prox}_{\gamma h}(x-\gamma g)}
  +\frac{1}{\gamma}\prn*{
    \operatorname{prox}_{\gamma h}(x-\gamma g)-x+\gamma g
  }.
\end{align}
By \cref{eq:prox-gradient-map}, the second term is
$-G_{\gamma h}(x,g)+g$.  Hence
\begin{align}
  G_{\gamma h}(x,g)-g
  \in
  \partial h\prn*{\operatorname{prox}_{\gamma h}(x-\gamma g)}.
\end{align}
Since $s\in\partial h(x)$, the two subgradient inequalities are
\begin{align}
  h\prn*{\operatorname{prox}_{\gamma h}(x-\gamma g)}
  &\geq h(x)+
  \inpr*{s,\operatorname{prox}_{\gamma h}(x-\gamma g)-x},\\
  h(x)
  &\geq h\prn*{\operatorname{prox}_{\gamma h}(x-\gamma g)}
  +\inpr*{
    G_{\gamma h}(x,g)-g,
    x-\operatorname{prox}_{\gamma h}(x-\gamma g)
  }.
\end{align}
Adding them gives
\begin{align}
  0
  &\leq
  \inpr*{
    G_{\gamma h}(x,g)-g-s,
    \operatorname{prox}_{\gamma h}(x-\gamma g)-x
  }\\
  &=-\gamma\inpr*{
    G_{\gamma h}(x,g)-g-s,
    G_{\gamma h}(x,g)
  },
\end{align}
where the equality follows from \cref{eq:prox-gradient-map}.  Therefore,
\begin{align}
  \nrm{G_{\gamma h}(x,g)}^2
  &\leq
  \inpr*{g+s,G_{\gamma h}(x,g)}
  \leq
  \nrm{g+s}\nrm{G_{\gamma h}(x,g)}.
\end{align}
If $\nrm{G_{\gamma h}(x,g)}>0$, division by this norm completes the proof.
Otherwise, the claim is immediate.
\end{proof}

\begin{lemma}\label{lem:shifted-local-minimizer}
Fix $x\in\operatorname{dom}h$, $g\in\R^d$, and $D>0$.  Let $u^*$ be a
minimizer of $\inpr{g,u}+h(x+u)-h(x)$ over $u\in\calU_D(x)$, and let
$y=x+u^*$.
Then, for any $\gamma>0$,
$\nrm{G_{\gamma h}(y,g)}\leq\calC_D(x,g)$.
\end{lemma}

\begin{proof}[Proof of \cref{lem:shifted-local-minimizer}]
The Lagrangian for $\nrm{u}^2/2\leq D^2/2$ gives some $\lambda\geq0$ such that
\begin{align}
  0&\in g+\partial h(x+u^*)+\lambda u^*,
  \qquad \lambda\prn*{\nrm{u^*}^2-D^2}=0.
\end{align}
Since $y=x+u^*$, there exists $s^*\in\partial h(y)$ such that
\begin{align}
  g+s^*+\lambda(y-x)&=0,
  \qquad \lambda\prn*{\nrm{y-x}-D}=0.
\end{align}
These conditions imply
\begin{align}
  \inpr{g+s^*,y-x}
  =-\lambda\nrm{y-x}^2
  =-D\lambda\nrm{y-x}
  =-D\nrm{g+s^*}. \label{eq:KKT_1}
\end{align}
The subgradient inequality
$h(y)-h(x)\leq\inpr{s^*,y-x}$ and $u^*=y-x$ give
\begin{align}
  \calC_D(x,g)
  &=-\frac{1}{D}\prn*{\inpr{g,u^*}+h(y)-h(x)}\\
  &\geq-\frac{1}{D}\inpr{g+s^*,y-x}\\
  &=\nrm{g+s^*} \tag{by \cref{eq:KKT_1}}\\
  &\geq\nrm{G_{\gamma h}(y,g)},
\end{align}
where the last inequality follows from \cref{lem:inclusion-to-proximal}.
\end{proof}

\begin{proof}[Proof of \cref{lem:shifted-conversion}]
Since $\nrm{y-x}\leq D$ and $\nu\geq2D$,
$\mathbb{B}(x,D)
  \subseteq\mathbb{B}(y,2D)
  \subseteq\mathbb{B}(y,\nu)$ and 
$g\in\partial_D f(x)
  \subseteq\partial_\nu f(y)$.
Then, by \cref{eq:prox-gradient-map},
\begin{align}
  \nrm{G_{\gamma h}(y,g)-G_{\gamma h}(y,\hat g)}
  &=\frac{1}{\gamma}\nrm*{
    \operatorname{prox}_{\gamma h}(y-\gamma\hat g)
    -\operatorname{prox}_{\gamma h}(y-\gamma g)}\\
  &\leq\frac{1}{\gamma}\nrm*{
    (y-\gamma\hat g)-(y-\gamma g)}\\
  &=\nrm{g-\hat g}
  \label{eq:pgmap-vector-lipschitz}
\end{align}
where the inequality follows from the nonexpansiveness of
$\operatorname{prox}_{\gamma h}$.

Therefore, we have
\begin{align}
  \min_{v\in\partial_\nu f(y)}
  \nrm{G_{\gamma h}(y,v)}
  &\leq\nrm{G_{\gamma h}(y,g)}\\
  &\leq\nrm{G_{\gamma h}(y,\hat g)}
  +\nrm{g-\hat g} \tag{by \cref{eq:pgmap-vector-lipschitz}}\\
  &\leq\calC_D(x,\hat g)+\nrm{g-\hat g} \tag{by \cref{lem:shifted-local-minimizer}}.
\end{align}
\end{proof}

%% file: section/appendix_zeroth_order.tex
\begin{proof}[Proof of \cref{lem:two-point-estimator}]
By Jensen's and the Cauchy--Schwarz inequalities and \cref{ass:zo-estimator}, for all $x,y$, we have
\begin{align}
  \abs{f(x)-f(y)}
  &\leq \E_\xi\brk*{\abs{F(x,\xi)-F(y,\xi)}}
  \leq \E_\xi\brk*{L(\xi)}\nrm{x-y}
  \leq L_0\nrm{x-y}.
\end{align}
The claims in \cref{eq:two-point-moment1,eq:two-point-moment2} follow directly
from \citet[Lemma~7]{kornowski2024zo}.

\end{proof}

\begin{proof}[Proof of \cref{thm:zeroth-order-complexity}]
We use the proof of \cref{thm:first-order-complexity} for $f_\rho+h$ with target
Goldstein radius $2D$, then use \cref{lem:smoothing-transfer}.
By \cref{lem:two-point-estimator}, $f$ is $L_0$-Lipschitz and
$\abs{f_\rho(x)-f(x)}\leq\rho L_0$ for every $x$.  Thus
\begin{align}
  \Delta_\rho\leq\Delta+2\rho L_0\leq\Delta+\delta L_0.
  \label{eq:smoothed-gap-bound}
\end{align}
{With $G=\sigma=G_{\mathrm{zo}}$ from \cref{lem:two-point-estimator}, the choices in \cref{thm:zeroth-order-complexity} satisfy}
\begin{align}
  T\geq\frac{4(2G_{\mathrm{zo}}+G_{\mathrm{zo}})^2}{\epsilon^2},
  \qquad
  K\geq\frac{2\Delta_\rho}{D\epsilon}.
\end{align}
Consequently, the proof of \cref{thm:first-order-complexity} gives
\begin{align}
  \E\brk*{\min_{g\in\partial_{2D}f_\rho(y_R)}
    \nrm{G_{\gamma h}(y_R,g)}}\leq\epsilon
\end{align}
for every $\gamma>0$.
Since $\rho+2D=\delta$, \cref{lem:smoothing-transfer} implies
$\partial_{2D}f_\rho(y_R)\subseteq\partial_\delta f(y_R)$, proving
the PGSP guarantee.

The resulting number of function-value queries is
{
\begin{align}
  N=2KT=O\prn*{\frac{dL_0^2(\Delta+\delta L_0)}{\delta\epsilon^3}}.
\end{align}
}
\end{proof}

%% file: section/appendix_smooth.tex
\section[Comparison and deferred proofs for the smooth setting]{Comparison and deferred proofs for \cref{sec:smooth}}\label{app:proofs-smooth}

\subsection{Comparison with existing bounds}\label{app:smooth-comparison}
\Cref{tab:smooth-comparison} compares \cref{thm:smooth-complexity,thm:smooth-zo-complexity} with existing oracle complexity bounds for finding an $\epsilon$-stationary point of $\Phi=f+h$ with $L_1$-smooth $f$, that is, a point $x$ with $\E\brk*{\nrm{G_{\gamma h}(x,\nabla f(x))}}\leq\epsilon$.
Bounds stated for $\E\brk*{\nrm{G_{\gamma h}(x,\nabla f(x))}^2}\leq\epsilon^2$ in the original papers imply this condition by Jensen's inequality.
The existing bounds hold for the step size $\gamma$ used by the respective method, whereas our bounds hold for every $\gamma>0$.
Mean-squared smoothness refers to the condition $\E_\xi\brk*{\nrm{\nabla F(x,\xi)-\nabla F(y,\xi)}^2}\leq L^2\nrm{x-y}^2$ for all $x,y\in\R^d$, which is stronger than the smoothness of $f$.

\begin{table}[t]
  \centering
  \small
  \caption{Oracle complexity for finding an $\epsilon$-stationary point of $\Phi=f+h$ with $L_1$-smooth $f$.
  The bounds show the dependence on $d$ and $\epsilon$.}
  \label{tab:smooth-comparison}
  \begin{tabular}{>{\raggedright\arraybackslash\leavevmode}p{5.4cm}ll>{\raggedright\arraybackslash}p{3.7cm}l}
    \toprule
    Reference & Oracle & Access & Assumption on the oracle & Complexity \\
    \midrule
    \citet{nesterov2013composite,ghadimi2016minibatch} & Deterministic & First-order & -- & $O(\epsilon^{-2})$ \\
    \citet{ghadimi2016minibatch} & Stochastic & First-order & Bounded variance & $O(\epsilon^{-4})$ \\
    \citet{wang2019spiderboost,pham2020proxsarah,trandinh2022hybrid} & Stochastic & First-order & Bounded variance, mean-squared smoothness & $O(\epsilon^{-3})$ \\
    \citet{ghadimi2016minibatch} & Stochastic & Zeroth-order & Each $F(\cdot,\xi)$ is smooth & $O(d\epsilon^{-4})$ \\
    \midrule
    \textbf{This work} (\cref{thm:smooth-complexity}) & Deterministic & First-order & -- & $O(\epsilon^{-2})$ \\
    \textbf{This work} (\cref{thm:smooth-complexity}) & Stochastic & First-order & Bounded variance & $O(\epsilon^{-4})$ \\
    \textbf{This work} (\cref{thm:smooth-zo-complexity}) & Stochastic & Zeroth-order & Each $F(\cdot,\xi)$ is Lipschitz & $O(d\epsilon^{-4})$ \\
    \bottomrule
  \end{tabular}
\end{table}

\subsection{Deferred proofs}\label{app:smooth-proofs}
We prove \cref{lem:smooth-stationarity} and the oracle complexity results from \cref{sec:smooth}.

\begin{proof}[Proof of \cref{lem:smooth-stationarity}]
For every $g\in\partial_\delta f(x)$, \cref{eq:pgmap-vector-lipschitz}
and the $L_1$-smoothness of $f$ give
\begin{align}
  \nrm{G_{\gamma h}(x,\nabla f(x))}
  \leq\nrm{G_{\gamma h}(x,g)}+\nrm{\nabla f(x)-g}
  \leq\nrm{G_{\gamma h}(x,g)}+L_1\delta.
\end{align}
Taking the minimum over $g\in\partial_\delta f(x)$ completes the proof.
\end{proof}

\begin{proof}[Proof of \cref{thm:smooth-complexity}]
Combining \cref{lem:shifted-conversion,lem:regret-gap-noise},
\cref{lem:smooth-stationarity} with $\delta=2D$, for every $\gamma>0$,
\begin{align}
  \E\brk*{\nrm{G_{\gamma h}(y_R,\nabla f(y_R))}}
  &\leq\frac{\Delta}{DK}+\frac{R_T}{DT}+\frac{\sigma}{\sqrt T}+2L_1D\\
  &\leq\frac{\Delta}{DK}+2L_1D
    +\frac{2\sqrt3\,(L_1D+\sigma)+\sigma}{\sqrt T} \tag{by \cref{thm:optimistic-comid-regret}}\\
  &\leq\frac{\Delta}{DK}+6L_1D+\frac{5\sigma}{\sqrt T}.
\end{align}
By the parameter choices in \cref{thm:smooth-complexity}, $6L_1D=\epsilon/2$,
$\Delta/(DK)\leq\epsilon/4$, and $5\sigma/\sqrt T\leq\epsilon/4$ gives 
$\E\brk*{\nrm{G_{\gamma h}(y_R,\nabla f(y_R))}} \leq \epsilon$.  
The number of stochastic gradient queries is
\begin{align}
  N=K(T+1)=O\prn*{\max\set*{
    \frac{L_1\Delta}{\epsilon^2},\frac{\sigma^2}{\epsilon^2},
    \frac{L_1\Delta\sigma^2}{\epsilon^4}}}.
\end{align}
\end{proof}

\begin{proof}[Proof of \cref{thm:smooth-zo-complexity}]
By $\nabla f_\rho(x)
=\E_{z\sim\unif(\mathbb B(0,1))}\brk*{\nabla f(x+\rho z)}$, we have
\begin{align}
  \nrm{\nabla f_\rho(x)-\nabla f_\rho(y)}
  &\leq\E_z\brk*{\nrm{\nabla f(x+\rho z)-\nabla f(y+\rho z)}}
    \leq L_1\nrm{x-y},\\
  \nrm{\nabla f_\rho(x)-\nabla f(x)}
  &\leq\E_z\brk*{\nrm{\nabla f(x+\rho z)-\nabla f(x)}}
    \leq L_1\rho,
\end{align}
where both bounds use the $L_1$-smoothness of $f$ and the second uses
$\nrm{z}\leq1$.  Since $\E_z\brk*{z}=0$, the quadratic Taylor bound gives
\begin{align}
  \abs{f_\rho(x)-f(x)}
  &=\abs*{\E_z\brk*{f(x+\rho z)-f(x)
    -\rho\inpr{\nabla f(x),z}}}
  \leq\frac{L_1\rho^2}{2}\E_z\brk*{\nrm{z}^2}
  \leq\frac{L_1\rho^2}{2}.
\end{align}
Thus the initial gap satisfies
\begin{align}
  \Delta_\rho\leq\Delta+L_1\rho^2
  =\Delta+\frac{\epsilon^2}{4L_1}.
\end{align}
{With $\sigma=G_{\mathrm{zo}}$ from \cref{lem:two-point-estimator}, the choices in \cref{thm:smooth-zo-complexity} satisfy}
{
\begin{align}
  T\geq\frac{1600G_{\mathrm{zo}}^2}{\epsilon^2},
  \qquad
  K\geq\frac{192L_1\Delta}{\epsilon^2}+48
  \geq\frac{192L_1\Delta_\rho}{\epsilon^2}.
\end{align}
}
{Consequently, the proof of \cref{thm:smooth-complexity} gives, for every $\gamma>0$,}
\begin{align}
  \E\brk*{\nrm{G_{\gamma h}(y_R,\nabla f_\rho(y_R))}}
  &\leq\frac{\Delta_\rho}{DK}+6L_1D
    +\frac{5G_{\mathrm{zo}}}{\sqrt T}\\
  &\leq\frac{\epsilon}{8}+\frac{\epsilon}{4}
    +\frac{\epsilon}{8}=\frac{\epsilon}{2}.
\end{align}
Using \cref{eq:pgmap-vector-lipschitz} and the gradient bias bound above,
we obtain
\begin{align}
  \E\brk*{\nrm{G_{\gamma h}(y_R,\nabla f(y_R))}}
  &\leq\E\brk*{\nrm{G_{\gamma h}(y_R,\nabla f_\rho(y_R))}}
    +\E\brk*{\nrm{\nabla f_\rho(y_R)-\nabla f(y_R)}}\\
  &\leq\frac{\epsilon}{2}+L_1\rho=\epsilon.
\end{align}
Each block uses $2(T+1)$ function-value queries, so
\begin{align}
  N=2K(T+1)
  &\leq2\prn*{49+\frac{192L_1\Delta}{\epsilon^2}}
    \prn*{2+\frac{1600G_{\mathrm{zo}}^2}{\epsilon^2}}\\
  &=O\prn*{\max\set*{
    \frac{L_1\Delta}{\epsilon^2},\frac{dL_0^2}{\epsilon^2},
    \frac{dL_0^2L_1\Delta}{\epsilon^4}}},
\end{align}
where the last equality uses
$G_{\mathrm{zo}}^2=16\sqrt{2\pi}\,dL_0^2$.
This completes the proof.
\end{proof}

%% file: section/appendix_moreau_pgsp.tex
\begingroup
\section{Relation between Moreau-envelope stationarity and PGSP}\label{app:moreau-pgsp}

Let $L_\delta>0$ be a Lipschitz constant of $\nabla f_\delta$. For $0<\gamma<1/L_\delta$, the Moreau envelope of $\Phi_\delta=f_\delta+h$ is
\begin{align}
  e_\gamma\Phi_\delta(x)
  \coloneqq\min_{z\in\R^d}\set*{\Phi_\delta(z)+\frac{\nrm{z-x}^2}{2\gamma}}.
\end{align}
Following \citet[Definition~3.1]{pougkakiotis2026inexact}, a point $x\in\operatorname{dom}h$ is an $\epsilon$-stationary point of $\Phi_\delta$ in the Moreau-envelope sense if $\nrm{\nabla e_\gamma\Phi_\delta(x)}\leq\epsilon$.

\begin{proposition}\label{prop:moreau-pgsp}
For $0<\gamma<1/L_\delta$ and $x\in\operatorname{dom}h$,
\begin{align}
  (1-\gamma L_\delta)\nrm{\nabla e_\gamma\Phi_\delta(x)}
  &\leq\nrm{G_{\gamma h}(x,\nabla f_\delta(x))}
  \leq(1+\gamma L_\delta)\nrm{\nabla e_\gamma\Phi_\delta(x)},
  \label{eq:moreau-prox-comparison}\\
  \min_{g\in\partial_\delta f(x)}\nrm{G_{\gamma h}(x,g)}
  &\leq(1+\gamma L_\delta)\nrm{\nabla e_\gamma\Phi_\delta(x)}.
  \label{eq:moreau-pgsp-transfer}
\end{align}
Thus, every $\epsilon$-stationary point of $\Phi_\delta$ in the Moreau-envelope sense is a $(\gamma,\delta,(1+\gamma L_\delta)\epsilon)$-PGSP of the original objective $\Phi=f+h$.
\end{proposition}

\begin{proof}
The $L_\delta$-Lipschitz continuity of $\nabla f_\delta$ and convexity of $h$ imply that $\Phi_\delta$ is $L_\delta$-weakly convex. Thus, the minimization defining $e_\gamma\Phi_\delta(x)$ is strongly convex and has a unique minimizer $z$, with
\begin{align}
  \nabla e_\gamma\Phi_\delta(x)=\frac{x-z}{\gamma}.
\end{align}
The optimality condition for $z$ is
\begin{align}
  0\in\nabla f_\delta(z)+\partial h(z)+\frac{z-x}{\gamma}
  =\partial h(z)+\frac{z-(x-\gamma\nabla f_\delta(z))}{\gamma}.
\end{align}
Hence, by \cref{eq:prox-definition},
\begin{align}
  z=\operatorname{prox}_{\gamma h}(x-\gamma\nabla f_\delta(z)).
\end{align}
Using \cref{eq:prox-gradient-map,eq:pgmap-vector-lipschitz} and the $L_\delta$-Lipschitz continuity of $\nabla f_\delta$, we obtain
\begin{align}
  \nrm{G_{\gamma h}(x,\nabla f_\delta(x))-\nabla e_\gamma\Phi_\delta(x)}
  &=\nrm{G_{\gamma h}(x,\nabla f_\delta(x))-G_{\gamma h}(x,\nabla f_\delta(z))}\\
  &\leq\nrm{\nabla f_\delta(x)-\nabla f_\delta(z)}\\
  &\leq L_\delta\nrm{x-z}
  =\gamma L_\delta\nrm{\nabla e_\gamma\Phi_\delta(x)}.
\end{align}
Thus, \cref{eq:moreau-prox-comparison} follows from the triangle inequalities, and \cref{eq:moreau-pgsp-transfer} follows from $\nabla f_\delta(x)\in\partial_\delta f(x)$ by \cref{lem:smoothing-transfer}.
\end{proof}

For exact sampled function values, \citet[Theorem~4.1]{pougkakiotis2026inexact} attain $\E\brk*{\nrm{\nabla e_\gamma\Phi_\delta(X)}^2}\leq\epsilon^2$ for their output $X$ using $O(d^{3/2}\delta^{-1}\epsilon^{-4})$ function-value queries, with $L_\delta=cL_0\sqrt d/\delta$ and $\gamma=(2L_\delta)^{-1}$, where $c>0$ is a universal constant.
Since $\gamma L_\delta=1/2$, \cref{eq:moreau-pgsp-transfer} and Jensen's inequality give
\begin{align}
  \E\brk*{\min_{g\in\partial_\delta f(X)}\nrm{G_{\gamma h}(X,g)}}
  &\leq\frac32\E\brk*{\nrm{\nabla e_\gamma\Phi_\delta(X)}}\\
  &\leq\frac32\sqrt{\E\brk*{\nrm{\nabla e_\gamma\Phi_\delta(X)}^2}}
  \leq\frac32\epsilon.
\end{align}
Thus, the same number of function-value queries yields a $(\gamma,\delta,3\epsilon/2)$-PGSP.

\endgroup

%% file: section/appendix_experimental_details.tex
\section{Experimental details}\label{app:experimental-details}

This section provides the experimental details of \Cref{sec:experiments}.
\paragraph{Problem setting.}

For the experiments, we use absolute loss for ReLU regression and ramp loss with $n=512$ and $d\in\set{16,256,1024}$.
The losses and common regularizer are
\begin{align}
  f_{\mathrm{ReLU}}(W)
  &=\frac1n\sum_{i=1}^n\abs*{\sum_{j=1}^{\sqrt d} v_j\max\set{0,\inpr{W_j,a_i}}-b_i},
  \label{eq:experiment-relu}\\
  f_{\mathrm{Ramp}}(x)
  &=\frac1n\sum_{i=1}^n\min\set{1,\max\set{0,1-b_i\inpr{a_i,x}}},
  \label{eq:experiment-ramp}\\
  h(x)&=\frac{0.02}{\sqrt d}\nrm{x}_1+\frac{0.01}{2}\nrm{x}^2
        +\ind_{[-0.5,0.5]^d}(x).
  \label{eq:experiment-regularizer}
\end{align}
Both problems have objective $\Phi=f+h$.
Both sample losses are $1$-Lipschitz.
The indicator function $\ind_C$ takes the value $0$ on $C$ and $+\infty$ outside $C$.
For absolute loss for ReLU regression, $x=\operatorname{vec}(W)$ with $W\in\R^{\sqrt d\times\sqrt d}$ and fixed output weights $v_j=(-1)^{j-1}/d^{1/4}$, without biases.

For data generation, $a_i$ are normalized Gaussian vectors with norm at most $1$. We draw the entries of $\bar W$ and $\bar w$ uniformly from $[-0.8,0.8]$, then zero each entry of $\bar W$ independently with probability $1/2$ and a randomly chosen half of the entries of $\bar w$. For absolute loss for ReLU regression, we set $b_i=\sum_{j=1}^{\sqrt d}v_j\max\set{0,\inpr{\bar W_j,a_i}}+\xi_i$, where $\xi_i$ are independent and uniform on $[-0.02,0.02]$. For ramp loss, we set $b_i=\operatorname{sign}(\inpr{a_i,\bar w})$ and independently flip each sign with probability $0.1$. Each dataset is fixed during optimization. Initial points are sampled by normalizing Gaussian vectors to norm $0.5$.
We evaluate \CompositeONC{} at its current anchor after each block and Chen's methods at their latest completed iterates.
We use a numerical upper bound on the stationarity measure $\min_{g\in\partial_\delta f(x)}\nrm{G_{\gamma h}(x,g)}$ with $\gamma=1$ and the composite objective value as evaluation measures. Evaluation queries are excluded from the query counts.

\begin{figure}[!t]
  \centering
  \includegraphics[width=0.4\textwidth]{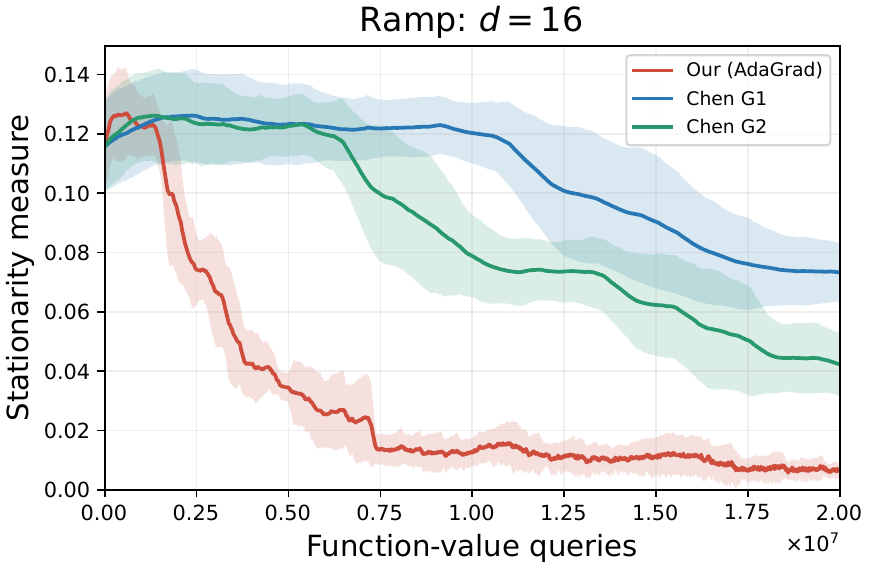}
  \includegraphics[width=0.4\textwidth]{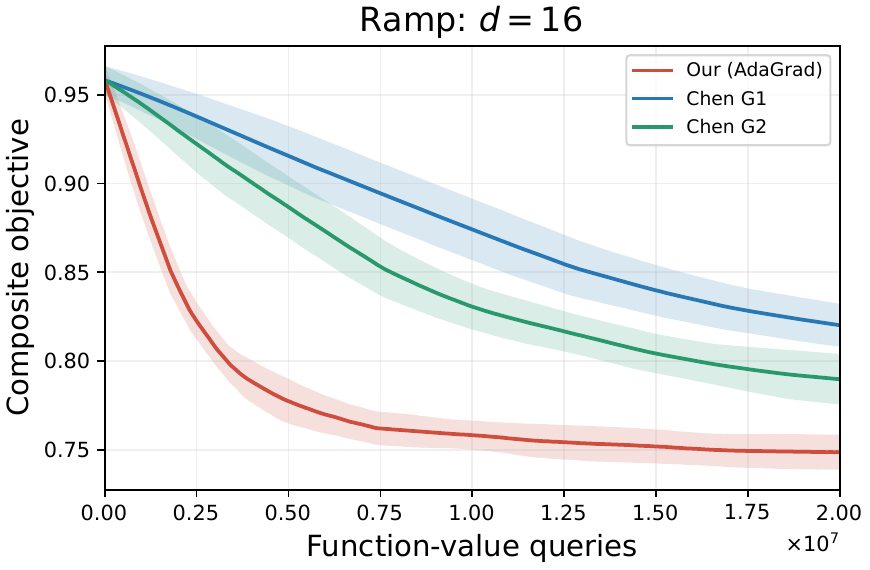}
  \includegraphics[width=0.4\textwidth]{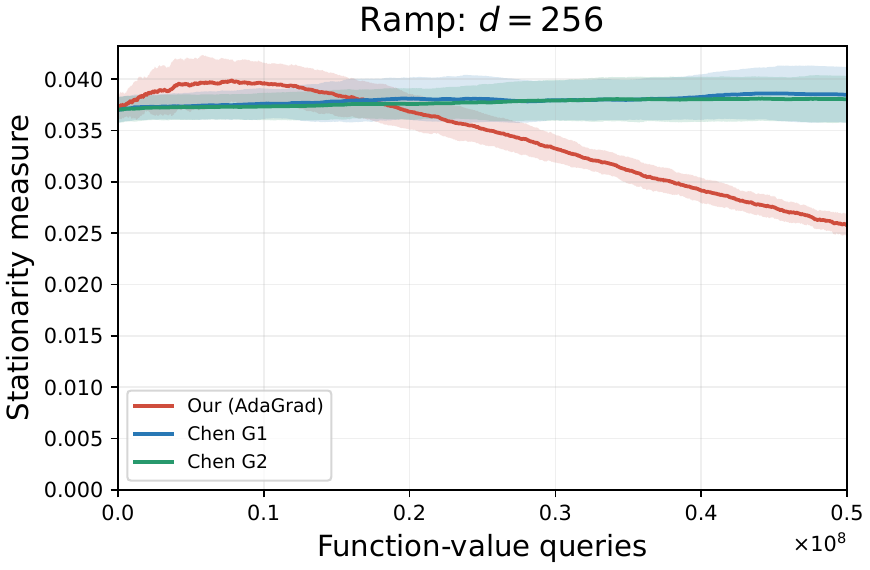}
  \includegraphics[width=0.4\textwidth]{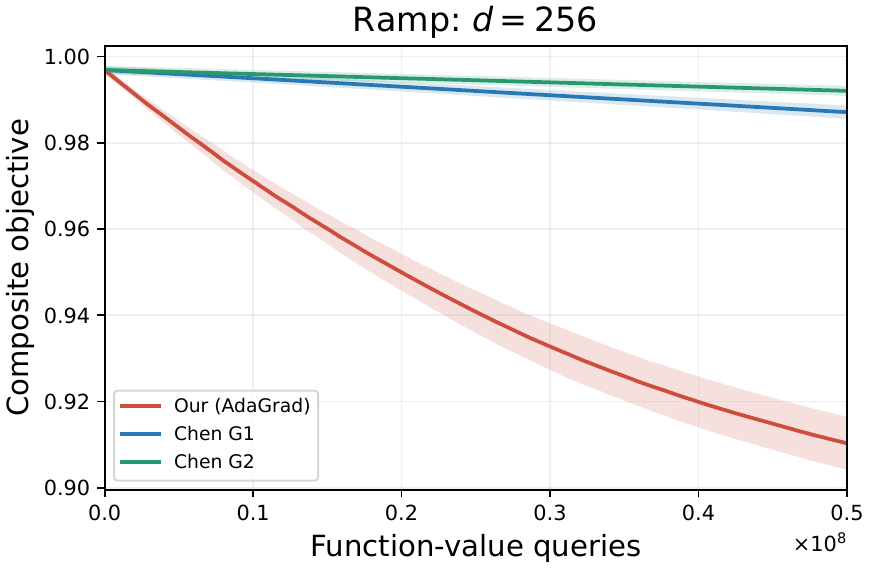}
  \includegraphics[width=0.4\textwidth]{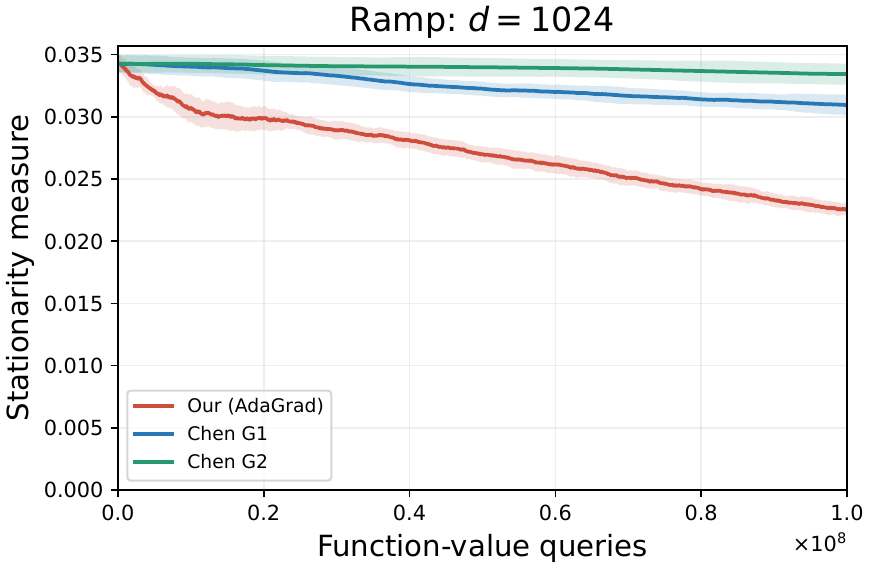}
  \includegraphics[width=0.4\textwidth]{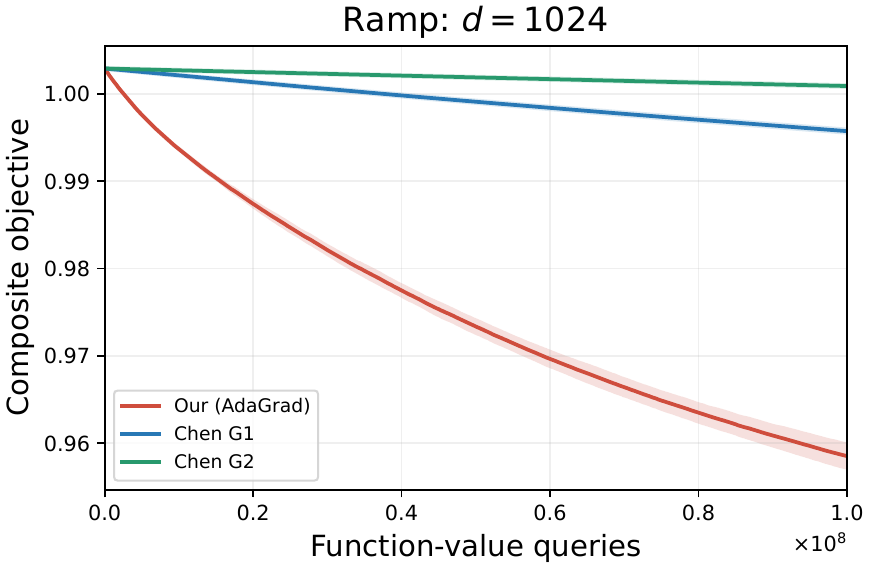}
  \caption{Stationarity measure $\min_{g\in\partial_\delta f(x)}\nrm{G_{\gamma h}(x,g)}$ (left) and composite objective values (right) for ramp loss
  at $d=16$ (top), $d=256$ (middle), and $d=1024$ (bottom).
  Curves and shaded bands show means and one standard deviation over ten trials.}
  \label{fig:experiment-ramp}
\end{figure}

\paragraph{Algorithm parameters.}
We fix $\delta=0.01$ and set the block length of \CompositeONC{} to $T=2500$. For Chen's methods, the batch sizes are $B=2500$ for G1 and $B_0=2500$, $B_1=50$ for G2, with refresh period $p=50$.
{For \CompositeONC{}, we set $D=\rho=\delta/3$. We use an adaptive step size like AdaGrad~\citep{duchi2011adagrad} based on the COMID regret bound in~\cref{eq:comid-regret-application},}
\begin{align}
  \eta_{k,t}=\frac{\sqrt{2}D}{\sqrt{1+\sum_{s=1}^t\nrm{g_{k,s}}^2}}.
\end{align}
{This step size gives $R_T=O(DG\sqrt T)$, retaining the order of the function-value query complexity in~\cref{thm:zeroth-order-complexity}.}
{For Chen's methods, we set $G=c=1$ and use the theoretical step sizes $\gamma_{\mathrm{G1}}=\frac{\delta}{\sqrt d}$ and $\gamma_{\mathrm{G2}}=\frac{\delta}{2(d+\sqrt d)}$ from \citet[Theorems~2 and 3]{chen2025}.}
The experiments were conducted using Python 3.12 on an Apple M1 Pro with 8 CPU cores and 16 GB of memory.
{For ramp loss, \cref{fig:experiment-ramp} likewise shows lower mean composite objective values for \CompositeONC{} over the course of optimization at all three dimensions. Its mean stationarity measure also becomes lower than those of both Chen methods as optimization proceeds.}

\paragraph{Computational cost and limitations.}
In our algorithm, the online learner solves convex subproblems at each round. The update rules are given by \cref{eq:comid-loss-update,eq:optimistic-comid} with fixed learning rates or by their variants with adaptive learning rates. Compared with minibatch methods, our algorithm incurs computational overhead from solving convex subproblems more frequently. Reducing computational costs and tuning learning rates for practical use are interesting directions. Since our main focus is improving oracle complexity, we leave them for future work.